\documentclass[twoside,11pt]{article}

\usepackage[preprint]{jmlr2e}
\usepackage[utf8]{inputenc}

\usepackage[british]{babel}
\usepackage[T1]{fontenc}
\usepackage{lmodern}
\usepackage{makecell, cellspace}
\usepackage{amsmath}
\usepackage[font=small,labelfont=bf]{caption}
\usepackage{amsfonts}
\usepackage{mathrsfs}
\usepackage{hyperref}
\usepackage{tikz}
\usetikzlibrary{decorations.pathreplacing, calligraphy}
\usepackage{mathtools}
\usepackage{multirow,tabularx}
\usepackage{xfrac}
\usepackage{enumitem}
\usepackage{booktabs}
\usepackage{listings}
\usepackage{verbatim}
\usepackage{printlen}
\usepackage{xcolor}
\AtBeginDocument{\colorlet{defaultcolor}{.}}
\newcommand{\colortxt}[2]{\color{#1} {#2} \color{defaultcolor}}
\usepackage{hhline}
\usepackage{float}
\usepackage[section]{algorithm}
\usepackage{algorithmic}
\usepackage{nicefrac}
\usepackage[textsize=scriptsize]{todonotes}
\usepackage{comment}

\renewcommand*{\vec}[1]{\ensuremath{\boldsymbol{#1}}}

\newtheorem{problem}{Problem}[section]

\renewcommand*{\phi}{\varphi}
\renewcommand*{\rho}{\varrho}
\newcommand*{\veps}{\varepsilon}
\renewcommand*{\subset}{\subseteq}

\renewcommand*{\star}{\ensuremath{\ast}}

\newcommand*{\dd}{\ensuremath{\mathrm{d}}}
\newcommand*{\dx}[1]{\ensuremath{\,\dd{#1}}}

\newcommand*{\RR}{\ensuremath{\mathbb{R}}}

\newcommand*{\NN}{\ensuremath{\mathbb{N}}}

\DeclareMathOperator{\supp}{supp}

\DeclareMathOperator{\ddiv}{div}

\DeclareMathOperator*{\esssup}{ess\,sup}

\newcommand*{\inner}[2]{\ensuremath{\left\langle #1, #2 \right\rangle}}
\newcommand*{\hnorm}[1]{\ensuremath{\left\Vert #1 \right\Vert_{H^1}}}
\newcommand*{\norm}[1]{\ensuremath{\left\Vert #1 \right\Vert}}

\newcommand*{\bfu}{\ensuremath{\mathbf{u}}}
\newcommand*{\bfx}{\ensuremath{\mathbf{x}}}
\newcommand*{\bfz}{\ensuremath{\mathbf{z}}}
\newcommand*{\bfy}{\ensuremath{\mathbf{y}}}
\newcommand*{\bfv}{\ensuremath{\mathbf{v}}}
\newcommand*{\bfw}{\ensuremath{\mathbf{w}}}

\newcommand*{\bff}{\ensuremath{\mathbf{f}}}

\newcommand*{\bfd}{\ensuremath{\mathbf{d}}}

\newcommand*{\bfr}{\ensuremath{\mathbf{r}}}
\newcommand*{\bfe}{\ensuremath{\mathbf{e}}}

\newcommand*{\bfkappa}{{\ensuremath{\boldsymbol{\kappa}}}}

\newcommand*{\tmeans}{\ensuremath{\boldsymbol{\Upsilon}}}

\newcommand*{\ATau}[1][{}]{\ensuremath{A^{#1}}}

\newcommand*{\Tri}{\ensuremath{\mathcal{T}}}
\newcommand*{\obfkappa}[1]{\overline{\bfkappa}^{(#1)}}
\newcommand*{\FAuk}{F}
 \newcommand*{\convts}{\star^{2\text{ts}}}
 \newcommand*{\convs}{\star^{2\text{s}}}
 \newcommand*{\convv}{\star}
 \let\epsilon\veps
 
\renewcommand{\vec}[1]{\boldsymbol{#1}} 
\newcommand{\NNparam}{\vec{\theta}}

\newcommand{\MLNet}{\textrm{ML-Net}}
\newcommand{\VNet}{\textrm{UNetSeq}}

\newcommand{\UNet}{\textrm{UNet}}

\newcommand{\MRH}{ \text{MR}H^1}
\newcommand{\Err}{\mathcal{E}_{\MRH}}
\newcommand{\ErrRef}{\mathcal{E}^{\text{ref}}_{\text{MR}H^1}}
\newcommand{\Errl}{\mathcal{E}_{\text{MR}L^2}}
\newcommand{\ErrRefl}{\mathcal{E}^{\text{ref}}_{\text{MR}L^2}}
\newcommand{\levelOutput}[1][\ell]{\bfz^{(#1)}}

\newcommand{\Cin}{C_{\text{in}}}
\newcommand{\Cout}{C_{\text{out}}}
\newcommand{\Width}{W}
\newcommand{\Height}{H}
\newcommand{\kernel}{K}
\newcommand{\Wk}{\Width_\kernel}

\newcommand{\Hin}{\Height_{\text{in}}}
\newcommand{\Win}{\Width_{\text{in}}}
\newcommand{\Hout}{\Height_{\text{out}}}
\newcommand{\Wout}{\Width_{\text{out}}}

\newcommand{\epsin}{(0, \nicefrac{1}{2})}
\newcommand{\apmult}{\widetilde\times}
\DeclareMathOperator{\neigh2d}{Neigh2D}
\DeclareMathOperator{\mg}{MG}

\newenvironment{proofof}[1]{\par\noindent{\bf #1\ }}{\hfill\BlackBox\\[2mm]}

\usepackage[width=.85\textwidth]{caption}

\ShortHeadings{Multilevel CNNs for Parametric PDEs}{Heiß, Gühring and Eigel}
\firstpageno{1}

\begin{document}
\title{Multilevel CNNs for Parametric PDEs}

\author{\name Cosmas Heiß \email cosmas.heiss@gmail.com \\
       \addr Department of Mathematics\\
       Technical University of Berlin\\
       Berlin, Germany
       \AND
       \name Ingo G{\"u}hring \email guehring@math.tu-berlin.de \\
       \addr Machine Learning Group\\
       Technical University of Berlin\\
       Berlin, Germany
       \AND
       \name Martin Eigel \email martin.eigel@wias-berlin.de \\
       \addr Weierstraß Institute\\
       Berlin, Germany
       }

\maketitle


\begin{abstract}

We combine concepts from multilevel solvers for partial differential equations (PDEs) with neural network based deep learning and propose a new methodology for the efficient numerical solution of high-dimensional parametric PDEs. An in-depth theoretical analysis shows that the proposed architecture is able to approximate multigrid V-cycles to arbitrary precision with the number of weights only depending logarithmically on the resolution of the finest mesh. As a consequence, approximation bounds for the solution of parametric PDEs by neural networks that are independent on the (stochastic) parameter dimension can be derived.

The performance of the proposed method is illustrated on high-dimensional parametric linear elliptic PDEs that are common benchmark problems in uncertainty quantification. We find substantial improvements over state-of-the-art deep learning-based solvers. As particularly challenging examples, random conductivity with high-dimensional non-affine Gaussian fields in 100 parameter dimensions and a random cookie problem are examined. Due to the multilevel structure of our method, the amount of training samples can be reduced on finer levels, hence significantly lowering the generation time for training data and the training time of our method.
\end{abstract}


\section{Introduction}
\label{sec:introduction}

The application of deep learning (DL) to many problems of the natural sciences has become one of the most promising emerging research areas, sometimes coined \textit{scientific machine learning} \citep{baldi2018biomedical, Sadowski2018, keith2021combining, noe2020machine, lu2021deepxde}.
In this work, we focus on the challenging problem of parametric partial differential equations (PDEs) that are used to describe complex physical systems e.g. in engineering applications and the natural sciences. The parameters determine the data or domain and are used to introduce uncertainties into the model, e.g. based on assumed or measured statistical properties. The resulting high-dimensional problems are analyzed and solved numerically in the research area of uncertainty quantification (UQ) \citep{lord2014introduction}. Such parametric models are of importance in automated design, weather forecasting, risk simulations, and the material sciences, to name but a few.

To make the problem setting concrete, for a possibly countably infinite parameter space $\Gamma\subset\mathbb R^{\mathbb N}$ and a spatial domain $D\subset \RR^d$, $d\in\{1,2,3\}$, we are concerned with learning the solution map $u:\Gamma\times \overline{D}\to\RR$ of the \emph{parametric stationary diffusion PDE} that satisfies the following equation for all parameters $\bfy\in\Gamma$
\begin{equation}
\label{eq:darcy}
    \begin{cases} 
        -\ddiv_x \kappa(\bfy, x)  \nabla_x u(\bfy, x) &= f(x) \quad\text{on } D, \\
        \hfill u(\bfy, x) &= 0 \qquad\text{on } \partial D,
    \end{cases}    
\end{equation}
where $\kappa:\Gamma \times D \to \RR$ is the parameterized diffusion coefficient describing the diffusivity, i.e. the media properties.
Note that differential operators act with respect to the physical variable $x$.

As a straight-forward but computationally ineffective solution, Equation~\eqref{eq:darcy} could be solved numerically for each $\bfy\in\Gamma$ individually, using e.g. the finite element (FE) method and a Monte Carlo approach to estimate statistical quantities of interest. More refined methods such as stochastic Galerkin \citep{EigelGittelson2014asgfem,EigelMarschall2020lognormal}, stochastic collocation \citep{stochastic_collocation, nobile2008sparse,ernst2018convergence} or least squares methods \citep{chkifa2015discrete,EigelTrunschke2019vmc,cohen2023near} exploit smoothness and low-dimensionality of the solution manifold $\{u(\bfy): \bfy \in \Gamma\}$ by projection or interpolation of a finite set of discrete solutions $\{u_h(\bfy_n)\}_{n=1}^N$.
After the numerical solution of Equation~\eqref{eq:darcy} at sample points $\bfy_n$ and the evaluation of the projection or interpolation scheme, solutions for arbitrary $\bfy\in \Gamma$ can quickly be obtained approximately simply by evaluating the constructed functional ``surrogate''. However, while compression techniques such as low-rank approximations or reduced basis methods foster the exploitation of low-dimensional structures, these approaches still become practically infeasible for large parameter dimensions rather quickly. Multilevel approaches represent another concept that can be applied beneficially for this problem class as was e.g. shown with multilevel stochastic collocation (MLSC) \citep{teckentrup2015} and multilevel quadrature \citep{harbrecht2016multilevel,ballani2016multilevel}. The central idea is to equilibrate the error components of the physical and the statistical approximations. This can be achieved by performing a frequency decomposition of the solution in the physical domain based on a hierarchy of nested FE spaces determined by appropriate spatial discretizations. An initial approximation on the coarsest level is successively refined by additive corrections on finer levels. Corrections generally contain strongly decreasing information for finer levels, which can be exploited in terms of a level dependent fineness of the stochastic discretization.

For the numerical solution of (usually non-parametric) PDEs, a plethora of NN architectures has been devised~\citep[see e.g.,][]{beck2020overview, berner2020numerically,raissi2019physics,sirignano2018dgm,weinan2017deep,ramabathiran2021spinn,kharazmi2019variational,yu2018deep,han2017deep,li2020fourier}. Similar to classical methods, DL-based solvers for parametric PDEs typically learn the mapping $\bfy \mapsto u_h(\bfy)$ in a supervised way from computed (or observed) samples  $\{u^h(\bfy_n)\}_{n=1}^N$ \citep[see e.g.,][]{geist2020numerical, HESTHAVEN201855, anandkumar2020neural}. This results in a computationally expensive data generation process, consisting of approximately computing the solutions at $\bfy_n$ and subsequently training the NN on this data. After training, only a cheap forward pass is necessary for the approximate solution of~\eqref{eq:darcy} at arbitrary $\bfy$. We note that in contrast to many classical machine learning problems, a dataset can be generated on the fly and in principle with arbitrary precision.

Complementing numerical results, significant progress has been made to understand the approximation power, also called \emph{expressivity}, of NNs for representing solutions of PDEs~\citep{kutyniok2019theoretical, schwab2019deep, berner2020analysis, mishra2019estimates, grohs2018proof, guhring2019error, guhring2021approx}.
These works show that NNs are at least as expressive as ``classical'' best-in-class methods like the ones mentioned before. However, numerical results do often not reflect this theoretical advantage~\citep{adcock2021gap}. Instead, the convergence of the error typically stagnates early during training and the accuracy of classical approaches is not reached \citep{geist2020numerical,kharazmi2019variational,lu2021deepxde,grossmann2023can}.

We propose a DL-based NN approach that overcomes the performance shortcomings of other NN architectures for parametric PDEs by borrowing concepts from MLSC and multigrid algorithms. Coarse approximation and successive corrections on finer levels are learned in a supervised way by sub-networks and combined for the final prediction. Analogously to multilevel Monte Carlo (MLMC) and MLSC methods, we show that the fineness of sampling the stochastic parameter space (in terms of the amount of training data per level) can be (reciprocally) equilibrated with the fineness of the spatial discretization. The developed architecture consists of a composition of UNets \citep{ronneberger_u_net}, which inherently are closely related to multiresolution approaches.

In the following, we summarize our main contributions.
\begin{itemize} 
    \item \emph{Theory:} We conduct an extensive analysis of the expressivity of UNet-like convolutional NNs (CNNs) for the solution of parametric PDEs. Our results show that under the uniform ellipticity assumption, the number of required parameters is independent on the parameter dimension with increasing accuracy. In detail, we show that common multilevel CNN architectures are able to approximate a V-cycle multigrid solver iteration with Richardson smoother \citep{richardson_iteration} with less than $\log(\nicefrac{1}{h})$ weights up to arbitrary accuracy, where $h$ is the mesh size of the underlying FE space (Theorem~\ref{thm:cnn_multigrid}).
    Using this result, we approximate the solution map of \eqref{eq:darcy} up to expected $H^1$-norm error $h$ by an NN with $\log(\nicefrac{1}{h})^2$ weights.
    \item \emph{Numerical Results:} We conduct an in-depth numerical study of our method and, more generally, UNet-based approaches on common benchmark problems for parametric PDEs in UQ.
   Strikingly, the tested methods show significant improvements over the state-of-the-art NN architectures for solving parametric PDEs.
   To reduce computing complexity during training data generation and training the NN, we shift the bulk of the training data to coarser levels, which allows our method to be scaled to finer grids, ultimatively resulting in higher accuracy.
\end{itemize}

\subsection{Related Work}
\cite{schwab2019deep,grohs2022deep} derived expressivity results for approximating the solution of parametric PDEs
based on a sparse generalized polynomial chaos representation of the solution known from UQ~\citep{schwab2011sparse,cohen2015approximation_long} and a central result for approximating polynomials by NNs~\citep{yarotsky2017error}.
\cite{kutyniok2019theoretical} translated a reduced basis method into an NN architecture. None of these results show bounds that are independent on the (stochastic) parameter dimension. In contrast, translated to our setting (and simplified), \cite{kutyniok2019theoretical} show an asymptotic upper bound of $h^{-2}\log(h^{-1})^p + p\log(h^{-1})^{3p}$ weights.
Numerical experiments are either not provided or do not support the positive theoretical results \citep{geist2020numerical}. For general overview of expressivity results for NNs, we refer the interested reader to~\citep{guehring2020expressivity}.

Related to our method and instructive for its derivation is the MLSC method presented in \citep{teckentrup2015} and also~\citep{harbrecht2016multilevel,ballani2016multilevel}.  Another related work is~\citep{lye_mishra_molinaro_2021}, where a training process with multilevel structure is presented. Under certain assumptions on the achieved accuracy of the optimization, an error analysis is carried out that could in part also be transferred to our architecture.

The connection between differential operators and convolutional layers is well known \citep[see][]{metaMGNet} and CNNs have already been applied to approximate parametric PDEs for problems like predicting the airflow around a given geometry in the works \citep{CNNFlowAerofoil, parametricCNNFlowCars}. In~\citep{bayesianCNNstochasticPDE} Bayesian CNNs were used to assess uncertainty in processes governed by stochastic PDEs. \citep{lagrangianFluidConvStuff} extend the concept to continuous convolutions in mesh-free Lagrangian approaches for solving problems in fluid dynamics. The concept of a multilevel decomposition has also been applied to CNNs to approximate solutions of PDEs or in~\citep{multilevelCNNnotPDE} for general generative modeling, or based on hierarchical matrix representations in~\citep{multilevelNetworkHirarchicalBases, multilevelNetworkHirarchicalMatrices}.

Such hierarchical matrix representations have even been extended to graph CNNs as multipole kernels to allow for a mesh-free approximation of the solutions to parametric PDEs~\citep{multilevelGraphPDEThingy}. This approach is similar to our method. However, the mathematical connection to classical multigrid algorithms is not explored. The kernels are based on fast multipole modeling without a direct link to the FE method.

\citet{MGNet} present a multilevel approach called MgNet, where various parts of a classical multigrid solver are replaced by NNs resulting in a UNet-like NN and use it for image classification. Meta-MgNet \citep{metaMGNet} is a meta-learning approach using MgNet to predict solutions for parametric PDEs and is most closely related to our method. The predicted solution is successively refined and stochastic parameters are incorporated by computing encodings of the parametric discretized operator on each level using additional smaller NNs. The output of these meta-networks then influences the smoothing operation in the MgNet architecture. However, in contrast to our method, Meta-MgNet still depends on the assembled operator matrix. Furthermore, while using an iterative approach yields more accurate approximations, the speedup promised by a single application of a CNN is mitigated and the authors report runtimes on the same order of magnitude as classical solvers. Lastly, the authors test their method on problems with uniform diffusivity using two or three degrees of freedom instead of a continuous diffusivity field. Our approach differs from the Meta-MgNet in the following ways: (i) We show that a general UNet is already able to approximate parametric differential operators without the need for a meta-network approach; (ii) we improve on the training procedure by incorporating a multilevel optimization scheme; (iii) we provide a theoretical complexity analysis; (iv) we provide experiments showing high accuracies with a single application of our method for a variety of random parameter fields.

\subsection{Outline}
In the first section we introduce the parametric Darcy model problem and its FE discretization. In Section~\ref{sec:multilevel_spaces}, we give a brief introduction to multigrid methods. In Section~\ref{sec:method}, we describe our DL-based solution approach. Section~\ref{sec:analysis} is concerned with the theoretical analysis of the expressivity of the presented architecture. Section~\ref{sec:experiments} contains our numerical study on several parametric PDE problems.
We discuss and summarize our findings in Section~\ref{sec:conclusion}. Implementation details and our proofs can be found in the appendix.

\section{Problem setting}
\label{sec:setting}

A standard model problem in UQ is the stationary linear diffusion equation \eqref{eq:darcy} where the diffusion coefficient $\kappa:\Gamma \to L^{\infty}(D)$ is a random field with finite variance determined by a possibly high-dimensional random parameter $\bfy \in \Gamma\subseteq \mathbb R^{\mathbb N}$, hence also parametrizing the solution.
The parameters are images of random vectors distributed according to some product measure $\pi=\otimes_{k\geq 1}\pi_1$, where $\pi_1$ usually either is the uniform (denoted by $U$) or the standard normal measure (denoted by $N$) on $\Gamma_1\subseteq\mathbb R$.
For each parameter realization $\bfy\in\Gamma$, the \emph{variational formulation} of the problem is given by: For a bounded Lipschitz domain $D\subset \RR^d$ and source term $f\in V^\ast=H^{-1}(D)$, find $u\in V:=H_0^1(D)$ such that for all test functions $w \in V$,
\begin{align}
\label{eq:variational darcy}
    \int_D \kappa(\bfy,x) \nabla u(x)\cdot\nabla w(x) \dx x = \int_D f(x)w(x)\dx x.
\end{align}
We assume that the above equation always exhibits a unique solution.
In fact, this holds true with high probability even for the case of unbounded (lognormal) coefficients, where uniform ellipticity of the bilinear form is not given.
This yields the solution operator $v$, mapping a parameter realization $\bfy \in \Gamma$ to its corresponding solution $v(\bfy)$:
\begin{equation*}
    v \colon \Gamma \to V,\qquad \bfy \mapsto v(\bfy).
\end{equation*}
In this work, we strive to develop an efficient and accurate NN approximation $\tilde{v}\colon \Gamma \to V$, which minimizes the expected $H^1$-distance 
\begin{equation*}
 \int_{\Gamma} \norm{\nabla(v(\bfy) - \tilde v(\bfy))}_{L^2(D)}^2 \dx\pi(\bfy).
\end{equation*}
With the assumed boundedness (with high probability) of the coefficient $\kappa$, this is equivalent to optimizing with respect to the canonical parameter-dependent norm.

We recall the truncated \emph{Karhunen-Lo\`eve expansion} (KLE) \citep{knio2006uncertainty,lord2014introduction} with stochastic parameter dimension $p$, which is a common parametric representation of random fields.
It is given by
\begin{equation}
\label{eq:diffusion_coefficient_expansion}
    \kappa(\bfy,x) = a_0(x) + \sum_{k=1}^p \bfy_k a_k(x).
\end{equation}
Here, the $p\in \NN$ basis functions $a_k \in L^2(D), 1\leq k \leq p$ are determined by assumed statistical properties of the field $\kappa$.
When Gaussian random fields are considered, they are completely characterized by their first two moments and the $a_k$ are the eigenfunctions of the covariance integral operator while the parameter vector $\bfy$ is the image of a standard normal random vector.
To become computationally feasible, a sufficient decay of the norms of these functions is required for the truncation to be reasonable.
Note that in the numerical experiments, we use a well-known ``artificial'' KLE.

\subsection{Discretization}
\label{sec:fem}
To solve the infinite dimensional variational formulation~\eqref{eq:variational darcy} numerically, a finite dimensional subspace has to be defined in which an approximation of the solution is computed.
As is common with PDEs, we use the FE method, which is based on a disjoint decomposition of the computational domain $D$ into simplices (also called elements or cells).
Setting the domain $D=[0,1]^2$, we assume a subdivision into congruent triangles such that $D$ is exactly represented by a uniform square mesh with identical connectivity at every node (see also Remark~\ref{rmk:uniform_square_mesh}).
The discrete space $V_h$ is then spanned by conforming P1 FE hat functions $\varphi_j:\RR^2\to\RR$, i.e., $V_h=\mathrm{span}\{\varphi_j\}_{j=1}^{\dim V_h}\subset V$ \citep[see e.g.,][]{braess2007finite,ciarlet_book}.
We write $u_h, w_h \in V_h$ for the discretized versions of $u, w \in V$ of Equation~\eqref{eq:variational darcy} with FE coefficient vectors $\bfu ,\bfw \in \RR^{\dim V_h}$ and denote by $\kappa_h$ the nodal interpolation of $\kappa$ such that ${\kappa_h(\bfy, \cdot) \in V_h}$ with coefficient vector $\bfkappa_\bfy\in \RR^{\dim V_h}$ for every $\bfy \in \Gamma$.
Hence,
\[
u_h = \sum_{i=1}^{\dim V_h} \bfu_i \phi_i, \quad w_h = \sum_{i=1}^{\dim V_h} \bfw_i \phi_i, \quad \kappa_h(\bfy, \cdot) = \sum_{i=1}^{\dim V_h} (\bfkappa_\bfy)_i \phi_i.
\]
The discretized version of the parametric Darcy problem~\eqref{eq:variational darcy} reads:
\begin{problem}[Parametric Darcy problem, discretized]
\label{def:darcy_discretized}
     For $\bfy \in \Gamma$, find $u_h \in V_h$ such that for all $w_h \in V_h$
    \begin{align*}
        a_{\bfy,h}(u_h, v_h) = f(w_h),
    \end{align*}
    where $a_{\bfy,h}(u_h, v_h)\coloneqq\int_D \kappa_h(\bfy, x) \nabla u_h(x)\cdot\nabla w_h(x) \dx x$.
    Using the FE coefficient vectors leads to the system of linear equations
\begin{equation}
\label{def:matrx_equation}
    A_\bfy \bfu = \bff,
\end{equation}
where $\bff\coloneqq (f(\phi_i))_{i=1}^{\dim V_h} = (\int_D f(x)\phi_i(x)\dx x)_{i=1}^{\dim V_h}$ and $A_\bfy\coloneqq (a_{\bfy,h}(\phi_j, \phi_i))_{i,j=1}^{\dim V_h}$.
\end{problem}

We assume that there always exists a unique solution to this problem and define the discretized parameter to solution operator by
\begin{equation}
\label{eq:discretized_solution_operator}
    v_h \colon \Gamma \to V_h, \qquad \bfy \mapsto v_h(\bfy),
\end{equation}
where $v_h(\bfy)$ solves Problem~\ref{def:darcy_discretized} for any $\bfy\in\Gamma$. Furthermore, $\bfv_h \colon \Gamma \to \RR^{\dim V_h}$ maps $\bfy$ to the FE coefficient vector of $v_h(\bfy)$.

\begin{remark}
We neglect the influence of the approximation of the diffusivity field $\kappa$ as $\kappa_h$ in our analysis since it would only introduce unnecessary technicalities. It is well-known from FE analysis how this has to be treated theoretically \citep[see e.g.,][Chapter 4]{ciarlet_book}. Notably, we always assume that the coefficient is discretized on the finest level of the multigrid algorithm that we analyze and evaluate in the experiments.
\end{remark}

\section{A primer on multigrid methods}
\label{sec:multilevel_spaces}

Our proposed method relies fundamentally on the notion of multigrid solvers for algebraic equations.
These iterative solvers exhibit an optimal complexity in the sense that they converge linearly in terms of degrees of freedom.
Their mathematical properties are well understood and they are commonly used for solving PDEs discretized, in particular, with FEs and a hierarchy of meshes \citep[see][]{braess2007finite,yserentant1993old,hackbusch2013multi}.
This section provides a brief primer on the multigrid setting.

The starting point is a hierarchy of nested spaces indexed by the level $\ell \in \{1,\ldots,L\}$ with finest level $L$, which determines the accuracy of the approximation. In our setting, this finest level is chosen in advance and is assumed to yield a sufficient FE accuracy.
Similar to the previous section, we define $V_\ell$ as the FE space of conforming P1 elements on the associated dyadic quasi-uniform triangulations $\mathcal{T}_\ell$ with mesh size $h_\ell\sim 2^{-\ell}$ and degrees of freedom $\dim V_\ell\sim 2^{\ell d}$.
This yields a sequence of nested FE spaces
$$
{V_1 \subset \ldots \subset V_L \subset H_0^1(D)}.
$$
These spaces are spanned by the the FE basis functions on the respective level, i.e., $V_\ell=\mathrm{span}\{\phi^{(\ell)}_j\}_{j=1}^{\dim V_\ell}$.
A frequency decomposition of the target function $u\in H_0^1(D)$ is obtained by approximating low frequency parts on the coarsest mesh $\mathcal T_1$ and subsequently calculating corrections for higher frequencies on finer meshes.

For each level $\ell=1,\ldots,L$ we define the discretized solution operator $v_\ell \colon \Gamma \to V_\ell$ as in~\eqref{eq:discretized_solution_operator}.
The correction $\hat v_\ell$ with respect to the next coarser level $\ell-1$ for $\ell\in \{2,\ldots,L\}$ is given by
\begin{equation}
\label{eq:correction_solution_operator}
    \hat{v}_\ell \colon \Gamma \to V_\ell,\qquad \bfy \mapsto v_\ell(\bfy) - v_{\ell-1}(\bfy).
\end{equation}
Additionally, we define $\bfv_\ell$ and $\hat{\bfv}_\ell$ as the functions mapping the parameter $\bfy$ to the corresponding coefficient vectors of $v_\ell(\bfy)$ and $\hat{v}_\ell(\bfy)$ in $V_\ell$.
If the solution of the considered PDE is sufficiently smooth, the corrections decay exponentially in the level, namely $\hnorm{\hat{v}_\ell(\bfy)} \lesssim 2^{-\ell} \norm{f}_\ast$ \citep[see][]{braess2007finite,yserentant1993old,hackbusch2013multi}.

Central to the multigrid method is a transfer of discrete functions between adjacent levels of the hierarchy of spaces.
This is achieved in terms of prolongation and restriction operators, the discrete versions of which are defined in the following.

\begin{definition}[Prolongation \& weighted restriction matrices]
\label{def:prolongation_restriction}
Let $L \in \NN$ and for some $\ell \in \{1,\ldots, L-1\}$ the spaces $V_\ell$ and $V_{\ell+1}$ be defined as above.
The \emph{prolongation matrix} $P_\ell \in \RR^{\dim V_{\ell+1} \times \dim V_\ell}$ is the matrix representation of the canonical embedding of $V_\ell$ into $V_{\ell+1}$ under the respective FE basis functions. $P_\ell^T \in \RR^{\dim V_\ell \times \dim V_{\ell+1}}$ defines the weighted restriction.
\end{definition}

The theoretical analysis of our multilevel NN architecture hinges on the multigrid V-cycle which is depicted in Algorithm~\ref{algo:v-cycle}. It exploits the observation that simple iterative solvers reduce high-frequency components of the residual. In addition to these so-called \emph{smoothing iterations}, low-frequency correction are recursively computed on coarser grids. In fact, if the approximation error by the recursive calls for the coarse grid correction is sufficiently small, a grid independent contraction rate is achieved \citep[see][]{hackbusch_v_cycle_proof,braess2007finite}.

For our analysis, we use a damped Richardson smoother~\citep{richardson_iteration}, which is rarely used in practice due to subpar convergence rates, but theoretically accessible because of its simplicity. For a suitable damping parameter $\omega>0$, the Richardson iteration for solving Equation~\eqref{def:matrx_equation} is given by
\begin{equation}
\begin{split}
    \bfu^{0} &:= \mathbf{0},\\
    \bfu^{n+1} &:= \bfu^{n} + \omega (\bff - A_{\bfy} \bfu^{n}).
\end{split}
\end{equation}

We denote by $\mg_{k_0, k}^m \colon \RR^{3 \times \dim V_h} \to \RR^{\dim V_h}$ the function mapping an initial guess, the diffusivity field and the right-hand side to the result of $m$ multigrid V-cycles (Algorithm~\ref{algo:v-cycle}) with $k_0$ smoothing iterations on the coarsest level and $k$ smoothing iterations on the finer levels.

Based on this iteration, we use a fundamental convergence result for the multigrid algorithm shown in~\citep[Thm.~4.2]{hackbusch_v_cycle_proof} together with the assumption of uniform ellipticity to get the uniform contraction property of the V-cycle multigrid algorithm. For $\bfx\in \RR^{\dim V_h}$ and the discretized operator $A\in \RR^{\dim V_h \times \dim V_h}$, we use the standard notation for the energy norm $\norm{\bfx}_{A} = \lVert{A^{1/2}\bfx}\rVert_{\ell^2}$.
\begin{theorem}
\label{thm:hackbusch_thm}
    Let $k\in \NN$. There exists $\omega>0$ and mesh-independent $k_0\in\NN$, $C<0$, such that the V-cycle iteration with $k$ pre- and post-smoothing iterations, $k_0$ iterations on the coarsest grid, and Richardson damping parameter $\omega$ applied to Problem~\ref{def:darcy_discretized} yields a grid independent contraction of the residual $\bfe^i_\bfy:=\mg_{k_0, k}^{i}(\mathbf{0}, \bfkappa_\bfy, \bff)-\bfv_h(\bfy)$ at the $i$-th iteration, namely
\begin{align*}
        \lVert{\bfe^{i+1}_\bfy}\rVert_{A_\bfy} \leq \mu(k) \lVert{\bfe^i_\bfy}\rVert_{A_\bfy},\quad \mu(k) \leq  \frac{C}{C + 2k} < 1,
    \end{align*}
     for all $\bfy\in\Gamma$.
\end{theorem}
In the next remark, we elaborate on the number of smoothing steps on the coarsest grid $k_0$.
\begin{remark}
\label{rmk:inexact_coarse_is_enough}
     In the V-cycle algorithm it is often assumed that the equation system solved on the coarsest level is solved exactly (by a direct solver). However, the proof of \citep[Thm.~4.2]{hackbusch_v_cycle_proof} shows that this is not necessary. In fact, a contraction of the residual smaller than $C / (C + 2k)$ on the coarsest level suffices. This shows that $k_0$ is independent on the fineness of the finest grid. For a rigorous derivation, we refer to~\cite[Theorem 4.4]{inexact_two_grid}. In our setting, it is possible to use the damped Richardson iteration for the correction on the coarsest grid as well, effectively carrying out additional smoothing steps.
\end{remark}

\begin{algorithm}
\caption{\texttt{V-cycle} function}
\label{algo:v-cycle}
\begin{algorithmic}[1]
\STATE \textbf{Input:} $\bfu$, $\bff$, $A$, $\ell$
    \FOR {$k$ pre-smoothing steps}
        \STATE $\bfu \gets \bfu + \omega (\bff - A \bfu)$ \COMMENT{perform smoothing steps}
    \ENDFOR
    \IF {$\ell = 1$}
    \STATE solve $A\bfu = \bff$ for $\bfu$ \COMMENT{coarsest grid step}
    \ELSE
    \STATE $P \gets P_{\ell - 1}$
    \STATE $\overline{\bfr} \gets P^\intercal (\bff - A \bfu)$ \COMMENT{compute restricted residual}
    \STATE $\overline{A} \gets P^\intercal A P$ \COMMENT{compute restricted operator}
    \STATE $\overline{\bfe} \gets \mathbf{0}$
    \STATE $\overline{\bfe} \gets \texttt{V-cycle}(\overline{\bfe}, \overline{\bfr}, \overline{A}, \ell-1)$ \COMMENT{solve for coarse correction $\overline{\bfe}$}
    \STATE $\bfu \gets \bfu + P \overline{\bfe}$ \COMMENT{add coarse correction}
    \ENDIF
    \FOR {$m$ post-smoothing steps}
        \STATE $\bfu \gets \bfu + \omega (\bff - A \bfu)$ \COMMENT{perform smoothing steps}
    \ENDFOR
\RETURN \bfu
\end{algorithmic}
\end{algorithm}

\section{Methodology of multilevel neural networks}
\label{sec:method}

This section gives an overview of our multilevel NN architecture for the efficient numerical solution of parametric PDEs. Some of the design choices are tailored specifically to the stationary diffusion equation~\eqref{eq:darcy} but can easily be adapted to other settings, in particular to different linear and presumably also nonlinear PDEs. Furthermore, in our analysis and experiments we fix the number of spatial dimensions to two. We note again that these results can be extended to more dimensions in a straightforward way e.g. by using higher-dimensional convolutions.

The general procedure of the proposed approach can be described as follows. Initially, a dataset of solutions is generated using classical finite element solvers for randomly drawn parameters (datapoints), i.e., realizations of coefficient fields leading to respective FE solutions. Then, for each datapoint a multilevel decomposition of the solution is generated as formalized in~\eqref{eq:correction_solution_operator}.
A NN is trained to output the coarse grid solution and the subsequent level corrections from the input parameters, which are given as coefficient vectors of the coefficient realizations on the respective FE space. The network prediction is obtained by summing up these contributions. Finally, to test the accuracy of the network, a set of independent FE solutions is generated and error metrics of the NN predictions are computed.

The centerpiece of our methodology is called $\MLNet$ (short for \emph{multilevel network}). For a prescribed number of levels $L\in\NN$, it consists of $L$ jointly trained subnetworks $\MLNet_\ell$. Each subnetwork is optimized to map its input -- the discretized parameter dependent diffusion coefficient $\kappa_{\bfy}$ and the output of the previous level $\levelOutput[\ell-1]$ -- to the level $\ell$ correction $\hat{\bfv}_\ell(\bfy)$ as in~\eqref{eq:correction_solution_operator}.
The $\MLNet_\ell$ subnetwork architecture is inspired by a cascade of multigrid V-cycles, which is implemented by a sequence of $\UNet$s \citep{ronneberger_u_net}. Using $\UNet$s in a sequential manner to correct outputs of previous modules successively is a well-established strategy in computer vision \citep[see e.g.,][]{ggmm22}.
We provide a theoretical foundation for this design choice in the next section. Since we assume the right-hand side $f \in H^{-1}(D)$ to be independent of $\bfy$, we do not include it as an additional input. More precisely,
\begin{equation*}
        \MLNet[\NNparam; L]:\RR^{\dim V_L}\to \bigotimes_{\ell=1}^L \RR^{\dim V_\ell}, \quad
        \vec{\kappa} \mapsto (\levelOutput)_{\ell=1}^L,
\end{equation*}
where $\levelOutput \in \RR^{\dim V_\ell}$ is the output of the level $\ell$ subnetwork and $\NNparam$ are the learnable parameters. For the subnetworks, we have
\begin{align*}
    &\MLNet_1[\NNparam_1; L]\colon\RR^{\dim V_L} \to \RR^{\dim V_1},\\
    &\MLNet_\ell[\NNparam_\ell; L]\colon\RR^{\dim V_L}\times \RR^{\dim V_{\ell - 1}}\to \RR^{\dim V_\ell},\quad \text{for } \ell=2,\ldots,L.
\end{align*}
The outputs are given by
\begin{align*}
    \levelOutput[1] &= \MLNet_1[\NNparam_1; L](\vec{\kappa}),\\
    \levelOutput &= \MLNet_\ell[\NNparam_\ell; L](\vec{\kappa}, \levelOutput[\ell - 1]),\quad \text{for } \ell=2,\ldots,L.
\end{align*}
Again, $\NNparam_\ell$ are the learnable parameters and $\NNparam = (\NNparam_\ell)_{\ell=1}^L$.

Each level consists of a sequence of $\UNet$s of depth $\ell$ with input dimension equal to $\dim V_\ell$ (arranged on a 2D grid). By coupling the $\UNet$ depth with the level, we ensure that the $\UNet$s on each level internally subsample the data to the coarse grid resolution $\dim V_1$ (again arranged on a 2D grid). Consequently, the number of parameters per $\UNet$ increases concurrently with the level. To distribute the compute effort (roughly) equally across levels, we decrease the number of $\UNet$s per level, which we denote by $\vec{R}_\ell$, where $\vec{R}\in\NN^L$. The NN architecture on level $\ell$ is given by
\begin{equation*}
    \MLNet_\ell[\NNparam_\ell; L](\vec{\kappa}, \vec{z}) := \left[\bigcirc_{r=1}^{\vec{R}_\ell}\UNet[\NNparam_{\ell, r}; \ell]([\downarrow\vec{\kappa},\cdot])\right](\uparrow\vec{z}),
\end{equation*}
where $\downarrow$ ($\uparrow$) denotes the up-sampling (down-sampling) operator to level $\ell$ resolution $\dim V_\ell$. These operators are motivated by the prolongation and restriction operators of Definition~\ref{def:prolongation_restriction}.
However, they are comprised of trainable strided (or transpose-strided) convolutions instead of an explicit construction. The parameters of the level $\ell$ subnetwork $\NNparam_\ell$ consist of the parameters of the learnable up- and down-sampling and the $\UNet$-sequence $(\NNparam_{\ell, r})_{r=1}^{\vec{R}_\ell}$.

Our complete approach to tackle a parametric PDE with $\MLNet$ consists of the following multi-step procedure:
\begin{itemize}[wide=0em]
    \item[\textbf{Step 1: Sampling and computing the data.}]
    We denote by $N\in \NN$ the training dataset size and sample parameter realizations $\bfy_1, \ldots, \bfy_N$ drawn according to $\pi$. For each parameter realization $\bfy_i$, $1\leq i\leq N$, the grid corrections $\hat{\bfv}_\ell(\bfy_i)$ on each level $1\leq \ell \leq L$ as in~\eqref{eq:correction_solution_operator} and the FE coefficient vector $\bfkappa_{\bfy_i} \in \RR^{\dim V_L}$ of $\kappa_L(\cdot, \bfy_i)$ on the finest level $L$ are computed. Note that the evaluation of the grid corrections requires the PDE solution for each parameter on the finest grid.
    
    \item[\textbf{Step 2: Training the $\MLNet$.}]
    For each level $\ell \in \{1,\ldots,L\}$, compute the normalization constants $\delta_\ell^2:= \frac{1}{N} \sum_{i=1}^N \norm{\hat{\bfv}_\ell(\bfy_i)}_2^2$ and the $H^1$ mass matrix
    $\vec{M}_\ell \in \RR^{\dim V_\ell \times \dim V_\ell}$ by
    \begin{align*}
        (\vec{M}_\ell)_{k j} := \inner{ \phi^{(\ell)}_k}{ \phi_j^{(\ell)}}_{H^1(D)} ,\quad\text{for}\quad k,j \in \{1,\ldots,\dim V_\ell\},
    \end{align*}
    where $\{\phi^{(\ell)}_1,\ldots, \phi^{(\ell)}_{\dim V_\ell}\}$ constitutes the conforming P1 FE basis of $V_\ell$. We optimize the parameters $\NNparam$ by approximately solving
    \begin{equation*}
        \min_{\NNparam}\sum_{\ell=1}^L \sum_{i=1}^N  \bfd_{\ell i}^T \vec{M}_\ell \bfd_{\ell i}, \quad\text{where}\quad \bfd_{\ell i} \coloneqq \big[\MLNet[\NNparam; L](\bfkappa_{\bfy_i})\big]_\ell - \frac{1}{\delta_\ell} \hat{\bfv}_\ell(\bfy_i)
    \end{equation*}
    with mini-batch stochastic gradient descent. Our loss computes the $H^1$-distance between the subnetwork output as a coefficient of the FE basis of $V_\ell$ and the normalized grid corrections directly on the coefficients for each level $\ell$, i.e.,
    \[
    \bfd_{\ell i}^T \vec{M}_\ell \bfd_{\ell i} = \norm{\sum_{j=1}^{\dim V_\ell} \left[\MLNet[\NNparam^*;L](\bfkappa_{\bfy_i})\right]_{\ell j} \phi_j^{(\ell)} - \frac{1}{\delta_\ell}\hat v_\ell(\bfy_i)}_{H^1(D)}.
    \]
    Normalization via $\delta_\ell$ assures that each level is weighted equally in the loss function.
    
    \item[\textbf{Step 3: Inference.}]
    Denoting the optimized parameters by $\NNparam^*$, the approximate solution (i.e., the NN prediction) for $\bfy\in\Gamma$ is computed by
    \begin{align*}
        v(\bfy) \approx \sum_{\ell=1}^L \delta_\ell \sum_{j=1}^{\dim V_\ell} \left[\MLNet[\NNparam^*;L](\bfkappa_\bfy)\right]_{\ell j} \phi_j^{(\ell)}.
    \end{align*}
\end{itemize}

\begin{figure}
    \centering
    \includegraphics[scale=0.13]{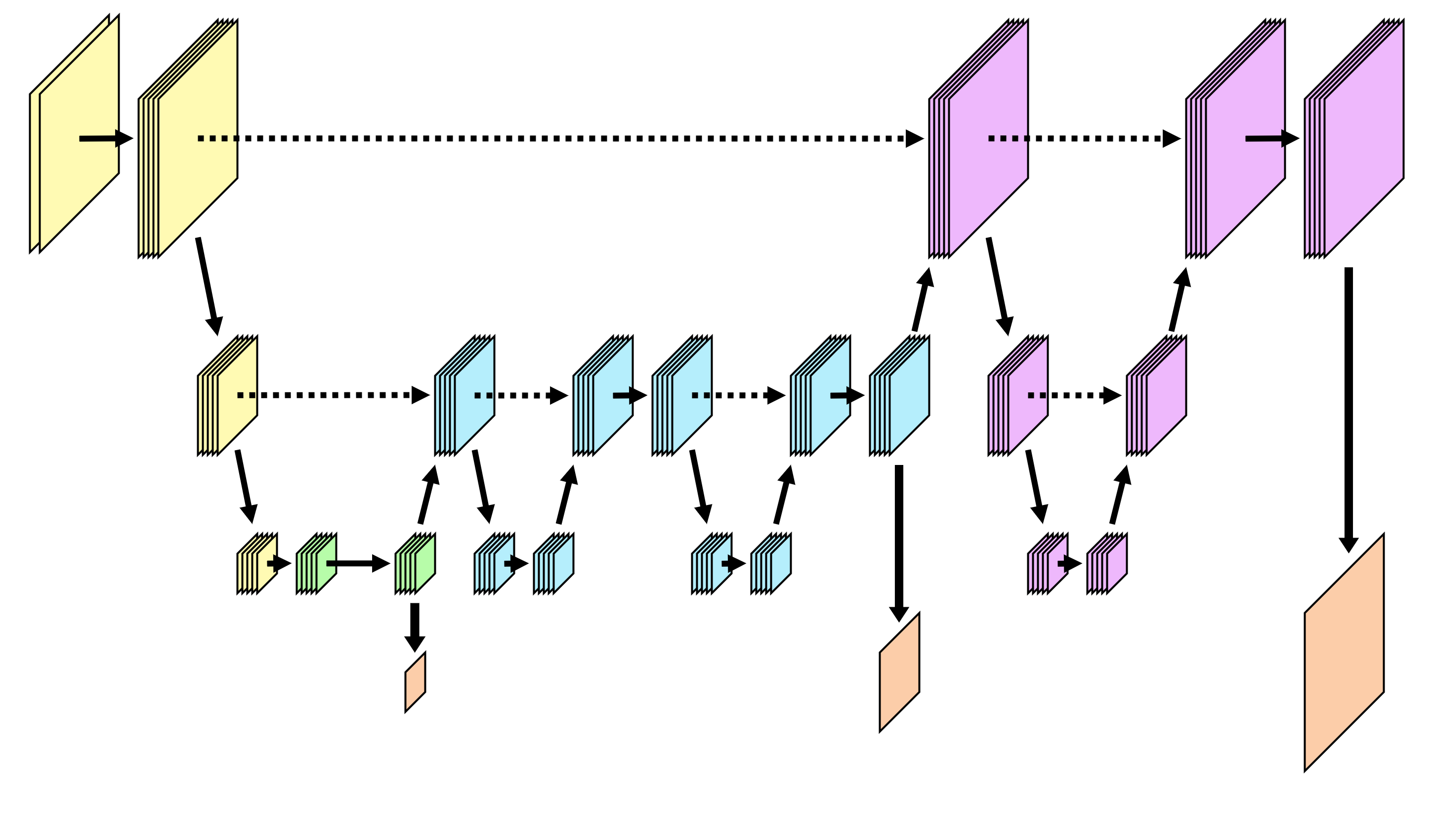}
    \caption{Illustration of the structure of the proposed $\MLNet$. The down-sampling cascade is shown in yellow. In this example, the network is constructed for $L=3$ with the green, cyan and magenta parts representing the individual stages on each level. The outputs are shown in orange. Solid arrows represent convolutions and dashed arrows skip connections.}
    \label{fig:fnet}
\end{figure}

In the following remark, a simpler baseline method is introduced that can be seen as a special case of $\MLNet$.
\begin{remark}\label{rmk:introduction_unetseq}
As a comparison and to assess the benefits of the described multilevel approach, we propose a simpler NN, which directly maps the discretized parameter dependent diffusion coefficient $\kappa_{\bfy}$ to the coefficients on the finest level. For this, we denote by $\VNet$ a sequence of $\UNet$s which can be seen as a variation of $\MLNet$ where $\vec{R}_\ell = 0$ for levels $\ell = 1, \ldots, L-1$, i.e., only the highest level sub-module is used. This architecture is a hybrid approach between single-level and multi-level because of the internal multi-resolution structure of the $\UNet$ blocks.
For this architecture, we choose the number of $\UNet$s such that the total parameter count roughly matches $\MLNet$.
\end{remark}

\section{Theoretical analysis}
\label{sec:analysis}
In this section, we show that UNet-like NN architectures are able to approximate multigrid V-cycles on $V_h$ up to arbitrary accuracy $\veps>0$ with the number of parameters independent on $\veps$ and growing only logarithmically in $1/h$ (Theorem~\ref{thm:cnn_multigrid}). We use this result in Corollary~\ref{crl:conv_multigrid_repr} to approximate the discrete solution $\bfv_h \colon \Gamma \to \RR^{\dim V_h}$ of the parametric PDE Problem~\ref{def:darcy_discretized} with arbitrary accuracy $\veps$ by a NN with the number of parameters bounded from above by $\log(\nicefrac{1}{h})\log(\nicefrac{1}{\veps })$ independent of the stochastic parameter dimension $p$. This is achieved by a NN that emulates a multigrid solver applied to each $\bfy$ individually and, thus, approximately computes the solution $\bfv_h(\bfy)$. Eventually, we derive similar approximation bounds for the $\MLNet$ architecture and point out computational advantages of $\MLNet$ over pure multigrid approximations of Corollary~\ref{crl:conv_multigrid_repr}. Note that while only the two-dimensional setting is considered in this section, the results can be translated to arbitrary spatial dimensions.

The rigorous mathematical analysis of NNs requires an extensive formalism, which can e.g., be found in~\citep{petersen2017optimal,guhring2019error,guhring2021approx}.
We assume that readers interested in the proofs are familiar with common techniques such as concatenation/parallelisation of NNs and the approximation of polynomials, the identity, and algebraic operations.
To make our presentation more accessible from a ``practitioner point of view'', we present a streamlined exposition which still includes sufficient details to follow the arguments of the proofs. All proofs can be found in Appendix~\ref{app:proofs}.

\emph{We consider a CNN as any computational graph consisting of (various kinds of) convolutions in combination with standard building blocks from common DL libraries (e.g., skip connections and pooling layers)}.
Note that this conception also includes the prominent UNet architecture \citep{ronneberger_u_net} and other UNet-like architectures with the same recursive structure. 

For our theoretical analysis, we consider activation functions $\varrho\in L^{\infty}_{\mathrm{loc}}(\RR)$ such that there exists $x_0\in \RR$ with $\varrho$ three times continuously differentiable in a neighborhood of some $x_0\in\RR$ and $\varrho''(x_0)\neq 0$. This includes many standard activation functions such as exponential linear unit, softsign, swish, and sigmoids. For (leaky) ReLUs a similar but more involved analysis would be possible, which we avoid for the sake of simplicity.

For a CNN $\Psi$, we denote the number of weights by~$M(\Psi)$. Since we consider the FE spaces $V_h$ in the two-dimensional setting on uniformly refined square meshes, FE coefficient vectors $\bfx \in \RR^{\dim V_h}$ can be viewed as two-dimensional arrays. Whenever $\bfx$ is processed by a CNN, we implicitly assume a 2D matrix representation. We define $\norm{\bfx}_{H^1}\coloneqq {\big\lVert\sum_{i=1}^{\dim V_h}\bfx_i \phi_i\big\rVert}_{H^1}$. For this section, we set $D=[0, 1]^2$ and assume uniform ellipticity for Problem~\ref{def:darcy_discretized}.

We start with the approximation of the solution of the multigrid V-cycle $\mg_{k_0, k}^m$ with initialization $\bfu_0$, diffusion coefficient $\bfkappa$, and right-hand side $\bff$ by UNet-like CNNs.

\begin{theorem}
\label{thm:cnn_multigrid}
    Let ${V_h \subset H_0^1(D)}$ be the P1 FE space on a uniform square mesh. Then there exists a constant $C>0$ such that for any $M, \veps>0$ and $m,k,k_0 \in \NN$ there exists a CNN $\Psi \colon \RR^{3 \times \dim V_h} \to \RR^{\dim V_h}$ with
    \begin{enumerate}[label=(\roman*)]
    \item 
    $\norm{\Psi(\bfu_0, \bfkappa, \bff) - \mg_{k_0, k}^m(\bfu_0, \bfkappa, \bff)}_{H^1(D)} \leq \veps$ for all $\bfkappa,\bfu_0,\bff \in [-M, M]^{\dim V_h}$,
    \item number of weights bounded by $M(\Psi) \leq C \left(k_0 + k \log\left(\frac{1}{h }\right)\right) m$. 
    \end{enumerate}
\end{theorem}
In the next remark, we provide more details about the architecture of $\Psi$ and its implications.

\begin{remark}
    The proof of Theorem~\ref{thm:cnn_multigrid} reveals that $\Psi$ resembles the concatenation of multiple UNet-like subnetworks, each approximating one V-cycle. We draw two conclusions from that:
    \begin{itemize}
        \item Generally, this supports the heuristic relation between UNets and multigrid methods with a rigorous mathematical analysis indicating a possible reason for their success in multiscale-related tasks e.g., found in medical imaging \citep{unetsurvey} or PDE-related problems \citep[see][]{unet_cascade_deblurring, unet_weather_stuff}.
        \item Together with our numerical results in the following section, this underlines the suitability of $\VNet$ (see Remark~\ref{rmk:introduction_unetseq}) for the solution of Problem~\ref{def:darcy_discretized}. This is made more specific in the next corollary.
    \end{itemize}
\end{remark}

The following corollary is an easy consequence of Theorem~\ref{thm:cnn_multigrid} for the parametric PDE Problem~\ref{def:darcy_discretized} by applying the multigrid solver to solve Equation~\eqref{eq:discretized_solution_operator} for each $\bfy\in \Gamma$ individually. We use that the required number of smoothing iterations in Theorem~\ref{thm:hackbusch_thm} is grid independent and choose the number of V-cycles as $m\approx \log(\nicefrac{1}{\veps)}$. An application of the triangular inequality yields the following result.
\begin{corollary}
\label{crl:conv_multigrid_repr}
    Consider the discretized Problem~\ref{def:darcy_discretized} with the conforming P1 FE space $V_h \subset H_0^1(D)$.
    Assume that $\bfkappa_\bfy$ is uniformly bounded over all $\bfy \in \Gamma$. Then there exists a constant $C>0$ such that for any $\veps>0$ there exists a CNN $\Psi \colon \RR^{2 \times \dim V_h} \to \RR^{\dim V_h}$ with
     \begin{enumerate}[label=(\roman*)]
    \item 
   $ \norm{\Psi(\bfkappa_\bfy, \bff) - \bfv_h(\bfy)}_{H^1(D)} \leq \veps \norm{f}_\ast$ for all $\bfy \in \Gamma$,
   \item number of weights bounded by $M(\Psi) \leq C\log\left(\frac{1}{h }\right)\log \left(\frac{1}{\veps }\right)$.
    \end{enumerate}
\end{corollary}

Setting the NN approximation error equal to the FE discretization error ($\veps = h$), we get the bound $M(\Psi)\leq C \log(\nicefrac{1}{h})^2$. This shows that the number of parameters at most grows polylogarithmically in $1/ h$ and is independent of the stochastic parameter dimension $p$. 

Finally, we provide complexity estimates for the $\MLNet$ architecture in the next corollary. Here, at each level of $\MLNet$ a multigrid V-cycle (on the respective level) is approximated by using Theorem~\ref{thm:cnn_multigrid}.
\begin{corollary}
\label{crl:fnet_repr_theorem}
    Consider the discretized Problem~\ref{def:darcy_discretized}.
    Let the spaces $V_1, \ldots, V_L$ be defined as in Section~\ref{sec:multilevel_spaces} for $L \in \NN$, $f \in H^{-1}(D)$ and its discretization denoted by $\bff \in \RR^{\dim V_L}$. Assume that the discretized diffusion coefficient $\bfkappa_\bfy\in \RR^{\dim V_L}$ is uniformly bounded over all $\bfy \in \Gamma$.
    Then, there exists a constant $C > 0$ such that for every $\veps >0$ there exists an $\MLNet$ $\Psi$ (as in Section~\ref{sec:method}) with
    \begin{enumerate}[label=(\roman*)]
        \item $\norm{\Psi(0, \bfkappa_\bfy, \bff) - \bfv_L(\bfy)}_{H^1(D)} \leq \veps \norm{f}_\ast$ for all $\bfy \in \Gamma$,
        \item $M(\Psi) \leq C L \log\left(\frac{1}{\veps} \right) + C L^2$.
    \end{enumerate}
\end{corollary}

In the following remark, we discuss the relation of the previous two corollaries.
\begin{remark}\label{rmk:computational_complexity}
Equilibrating approximation and discretization error, i.e., setting $\veps = h = 2^{-L}$, we make two observations:
\begin{itemize}
    \item Both, Corollary~\ref{crl:conv_multigrid_repr} and \ref{crl:fnet_repr_theorem}, yield an upper bound for the number weights of $M(\Psi) \leq C L^2$. We conclude that from our upper bounds no advantage in terms of expressiveness for using $\MLNet$ over $\VNet$ is evident. In Section \ref{sec:results} we show that our numerical experiments support this finding.
    \item Despite a similar number of parameters in relation to the approximation accuracy, a forward pass of ML-net is more efficient than a forward pass of $\VNet$. This is due to the multilevel structure of the $\MLNet$ that shifts the computational load more evenly across resolutions. Note that the number of parameters of a convolutional filter is independent on the input resolution, whereas the computational complexity of convolving the input with the filter is not. In detail, the number of operations for a forward pass of $\MLNet$ is $\mathcal{O}(2^{2L}) = \mathcal{O}(h^{-2})$, while for $\VNet$, we have $\mathcal{O}(L 2^{2L}) = \mathcal{O}(\log(h^{-1}) h^{-2})$ computations. This is one advantage of $\MLNet$ that results in shorter training times (see Section~\ref{sec:experiments}).
\end{itemize}
\end{remark}

\section{Numerical results}
\label{sec:experiments}
This section is concerned with the evaluation of $\MLNet$ and the simpler $\VNet$ on several common benchmark problems in terms of the mean relative $H^1$-distance of the predicted solution.
Moreover, the performance of $\MLNet$ is tested with a declining number of training samples for successive grid levels.
This is motivated by multilevel theory and leads to a beneficial training complexity (see Remark~\ref{rmk:computational_complexity} for a more detailed evaluation).

We assess the performance of our methods for different choices of varying computational complexity for the coefficient functions $a_k$, comprising the diffusivity field $\kappa$ by~\eqref{eq:diffusion_coefficient_expansion} and different parameter distributions $\pi$. For all test cases, we choose the unit square domain $D=[0,1]^2$, $a_0(x)\coloneqq a_0\in\RR$ constant and the right-hand side $f\equiv 1$.
The test problems are defined as follows.
\begin{enumerate}
\label{enum:kappa_choices}
    \item \textbf{Uniform smooth field.} Let $\Gamma = [-1,1]^{p}$ and $\pi = U[-1,1]^p$.
    Moreover, let $a_0 := 1$ and choose $a_k$ as planar Fourier modes as described in \citep{eigel2019variational} and \citep{geist2020numerical} with decay $\norm{a_k}_\infty = 0.1k^{-2}$ for $1\leq k\leq p$.
    
    \item \textbf{Log-normal smooth field.} Let $\Gamma = \mathbb{R}^{p}$ and $\pi = N(\mathbf{0},\text{Id}_p)$. Define $\kappa(\bfy, x) := \exp(\tilde{\kappa}(\bfy, x))$, where $\tilde \kappa$ equals $\kappa$ from the uniform case with $a_0 := 0$.
    
    \item \textbf{Cookie problem with fixed radii inclusions.} Let $\Gamma = [-1,1]^{p}$, with $p$ having a natural square root and $\pi = U[-1,1]^p$. Moreover, define $a_0 := 0.1$ and $a_k(x) = \mathcal{X}_{D_k}(x)$, where $D_k$ are disks centered in a cubic lattice with a fixed radius $r=0.3p^{-1/2}$ for $1\leq k \leq p$.
    
    \item \textbf{Cookie problem with variable radii inclusions.} 
    As an extension of the cookie problem, we also assume the disk radii to vary. To this end, let $p'$ be the (quadratic) number of cookies, $p=2p'$, $\Gamma = [-1,1]^{p}$ and $\pi = U[-1,1]^p$. The diffusion coefficient is defined by
    \begin{align*}
        \kappa(\bfy,x) := a_0 + \sum_{k=1}^{p'} \bfy_{2k-1} \mathcal{X}_{\mathscr{D}_{k, r(\bfy_{2k})}}(x),
    \end{align*}
    where the $\mathscr{D}_{k, r}$ are disks centered in a cubic lattice with parameter dependent radius~$r$. We choose $r(y)$ for $y\in [-1, 1]$ such that the radii are uniformly distributed in $[0.5 (p')^{-1/2},0.9 (p')^{-1/2}]$.
\end{enumerate}
Figure~\ref{fig:example_diff_coeffs} depicts a random realization for each of the 4 problem cases. Note that not all of the above problems satisfy the theoretical assumptions for the results in Section~\ref{sec:analysis}. Notably, the log-normal case is not uniformly elliptic and our theoretical guarantees thus only hold with (arbitrarily) high probability as indicated above. Moreover, in case of the cookie problem with variable radii, the diffusion coefficient depends non-linearly on the parameter vector since also the basis functions $a_k$ are parameter dependent. We include these cases to study our methodology with diverse and particularly challenging setups.

\begin{figure}
    \centering
    \includegraphics[width=0.2435 \textwidth]{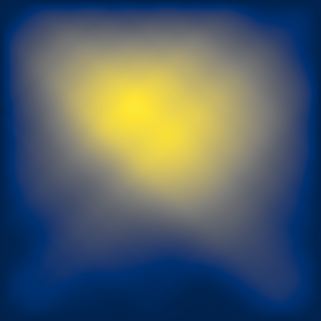}
    \includegraphics[width=0.2435 \textwidth]{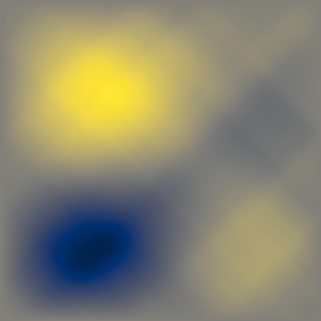}
    \includegraphics[width=0.2435 \textwidth]{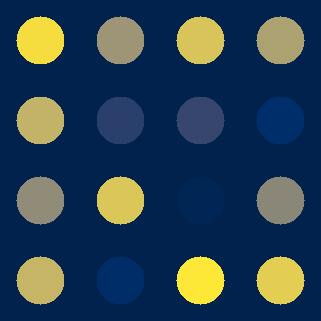}
    \includegraphics[width=0.2435 \textwidth]{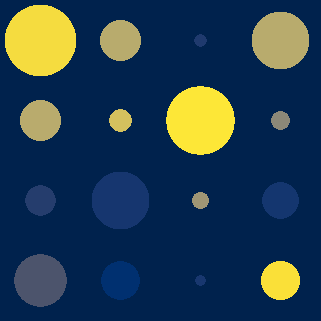}
    \caption{Example realizations of $\kappa(\cdot, \bfy)$ for the four test cases. From left to right: uniform case ($p=100$), log-normal case ($p=100$), cookie problem ($p=16$), cookie problem with variable radii ($p=32$). The colormaps are normalized for visibility.}
    \label{fig:example_diff_coeffs}
\end{figure}

We use the open source package \texttt{FEniCS} \citep{fenics_dolfin} with the GMRES \citep{gmres_paper} solver for carrying out the FE simulations to generate training data and \texttt{PyTorch} \citep{pytorch_cite} for the NN models.

If not stated otherwise, we use a training dataset with $10^4$ samples and $1024$ samples for independent validation computed from i.i.d. parameter vectors. For testing, the performance of the methods is evaluated on $1024$ independently generated test samples.

For the multilevel decomposition, we consider $L=7$ resolution levels starting with a coarse resolution of $5\times 5$ cells. Iterative uniform mesh refinement with factor $\eta=2$ results in $320 \times 320$ cells on the finest level. To reduce computational complexity, we only compute solutions on the finest level and derive the coarser solutions as well as the corrections in~\eqref{eq:correction_solution_operator} using nodal interpolation. Additionally, we compute reference test solutions on a high-fidelity mesh with $1280 \times 1280$ cells, resulting in about $1.6\times 10^6$ degrees of freedom.

For $\VNet$, we choose the number of $\UNet$s such that the total parameter count roughly matches the $5.6 \times 10^6$ parameters of $\MLNet$. For training, we use the Adam~\citep{adam_paper} optimizer and train for $200$ epochs with learning rate decay.
After each run, we evaluate the model with the best validation performance. Training took approximately $33$ GPU-hours on an NVIDIA Tesla P100 for $\MLNet$ and $46$ GPU-hours for $\VNet$ (see also Remark~\ref{rmk:computational_complexity}).

\begin{figure}
    \centering
    \includegraphics[scale=1.0]{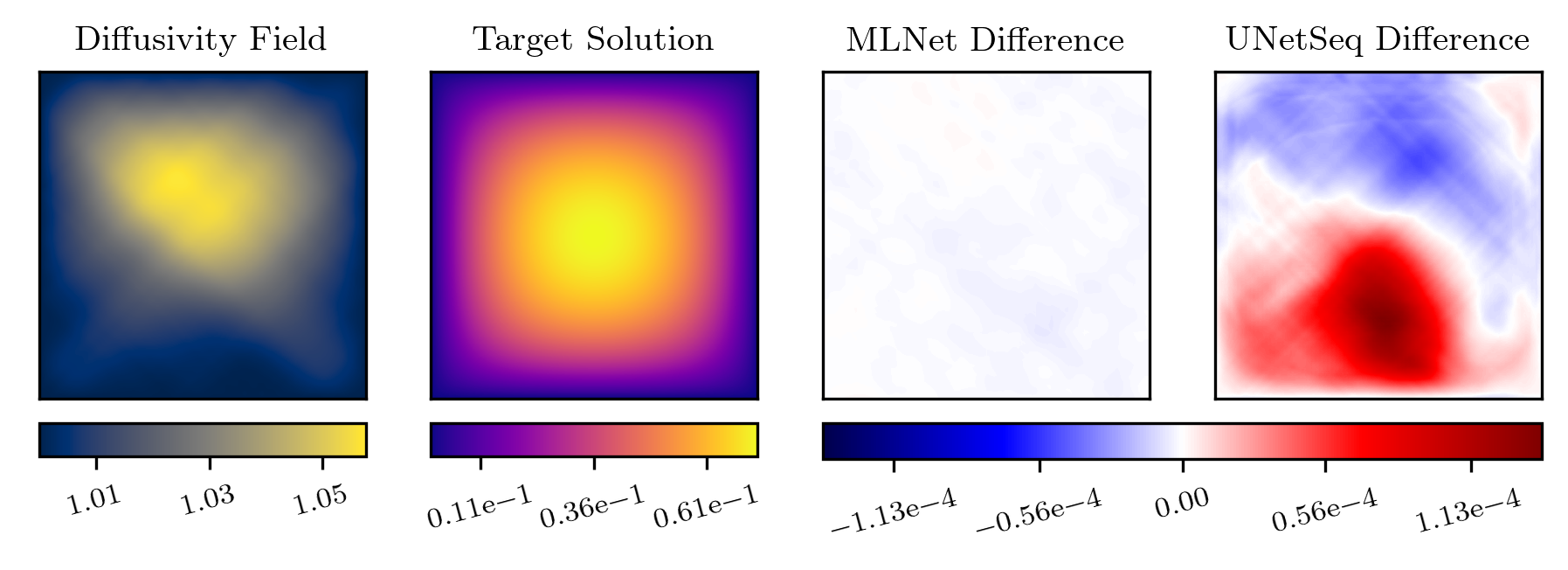}
    \includegraphics[scale=1.0]{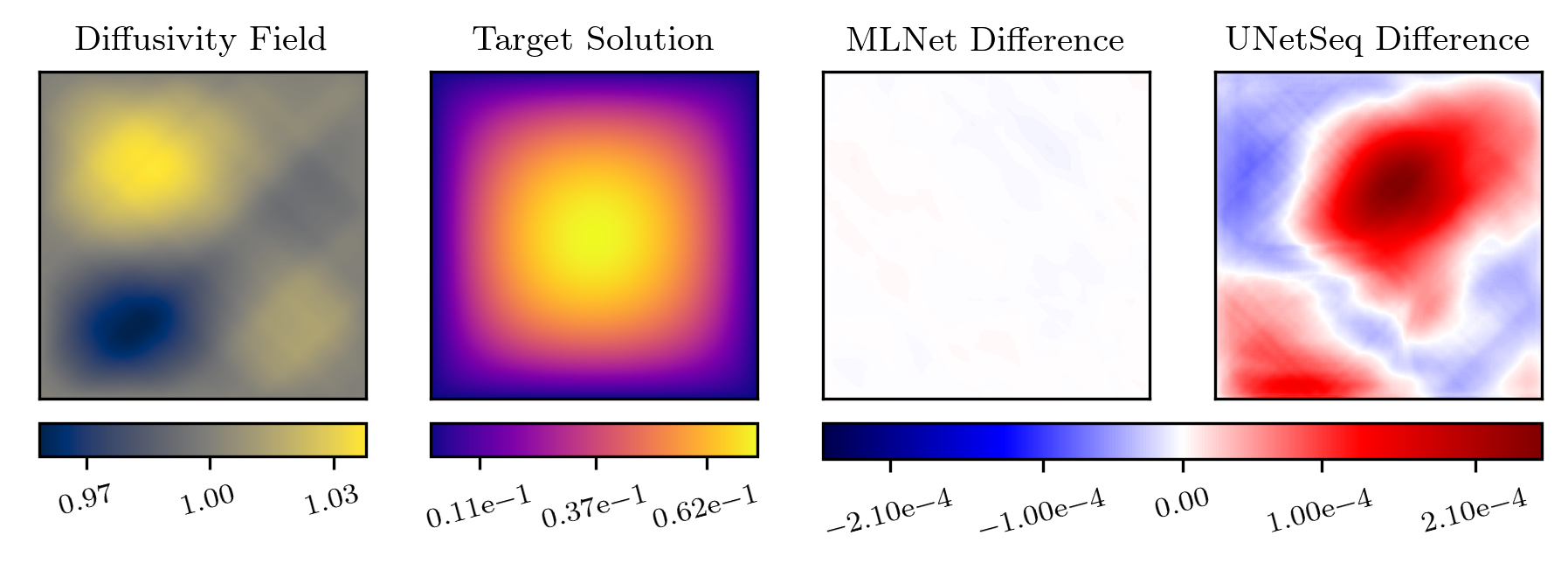}
    \includegraphics[scale=1.0]{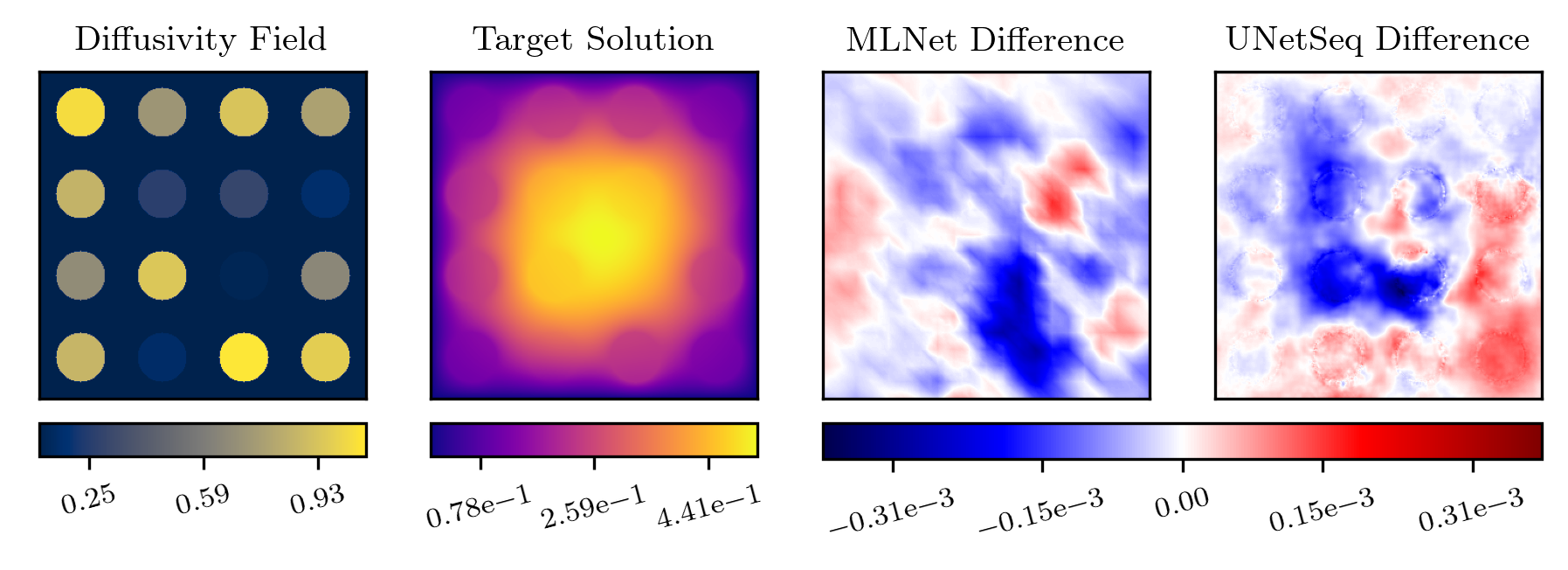}
    \caption{Realizations for different benchmark problems (top to bottom: uniform $p=100$, log-normal $p=100$, cookie fixed). Left-hand side columns show field realizations and respective solutions. Right-hand side pictures the error of $\MLNet$ and $\VNet$ predictions w.r.t.\ exact target solutions.}
    \label{fig:examples_network_differences}
\end{figure}

\subsection{Error metrics}
\label{sec:error metrics}
We examine the mean relative $H^1$ error (MR$H^1$) and mean relative $L^2$ error (MR$L^2$) with respect to the solutions on the finest resolution level $L$ as well as the high-fidelity reference solutions. For parameters $\bfy_1,\ldots,\bfy_N \in \Gamma$, predictions ${u}_1,\ldots,{u}_N \in V_L$, and the solution operators $v_L \colon \Gamma \to V_L$ and $v_{\text{ref}} \colon \Gamma \to H^1(D)$ mapping to the discrete space on the finest grid and the reference grid, respectively, we define
\begin{align*}
    \mathcal{E}_{\text{MR}\star} := \sqrt{\frac{\sum_{i=1}^N \norm{u_i - v_L(\bfy_i)}_\star^2}{\sum_{i=1}^N \norm{v_L(\bfy_i)}_\star^2}}\quad\text{and}\quad \mathcal{E}^{\text{ref}}_{\text{MR}\star} := \sqrt{\frac{\sum_{i=1}^N \norm{u_i - v_{\text{ref}}(\bfy_i)}_\star^2}{\sum_{i=1}^N \norm{v_{\text{ref}}(\bfy_i)}_\star^2}},
\end{align*}
with $\star \in \{H^1, L^2\}$.
\subsection{Results for the test cases}
\label{sec:results}

Tables~\ref{tab:results_mrh1} and~\ref{tab:results_ref_mrh1} depict the $\Err$ and $\ErrRef$ for $\MLNet$ and $\VNet$ for the numerical test cases described above with varying stochastic parameter dimensions. $\Errl$ and $\ErrRefl$ are shown in Tables~\ref{tab:results_mrl2} and \ref{tab:results_ref_mrl2} in Appendix~\ref{sec:add_tables}. A random selection of solutions and NN predictions is also shown in Figure~\ref{fig:examples_network_differences}.

In comparison to previously reported performances of DL-based approaches for these benchmark problems in  \citep{geist2020numerical, furier_nn_operators,lu2021deepxde,grossmann2023can}, we observe an improvement of one to two orders of magnitude with our methodology. 

\citet{geist2020numerical} observed a strong dependence of the performance of their DL-based method on the stochastic parameter dimension. In  contrast, both $\MLNet$ and $\VNet$ perform very consistently with increasing parameter dimensions although the problems become more involved. This is in line with the parameter independent bounds for CNNs derived in Section~\ref{sec:analysis}.

$\mathcal{E}_{\text{MR}H^1}$ is generally significantly lower than $\mathcal{E}^{\text{ref}}_{\text{MR}H^1}$. This illustrates that for these cases the NN approximation error on the finest grid $L=7$ can be several magnitudes lower than the FE approximation error. Conversely, the discrepancy between $\Errl$ and $\ErrRefl$ seen in Tables~\ref{tab:results_mrl2} and \ref{tab:results_ref_mrl2} is far less pronounced.

Comparing $\MLNet$ to $\VNet$, we observe that both models generally exhibit a comparable performance. Lower variances in model accuracy for $\MLNet$ indicate an improved training stability when using a suitable multilevel decomposition.

\begin{table}
    \centering
    \begingroup
    \renewcommand\cellalign{l}
    \setcellgapes{0.2ex}\makegapedcells
    \begin{tabular}{
    l c*{4}{>{\hspace*{0.0mm}}c<{\hspace*{0.0mm}}}}
        \toprule
        \multirow{2}[2]{*}{problem} & \multirow{2}[2]{*}{\makecell{parameter\\dimension $p$}} & \multicolumn{2}{c}{$\Err$} \\
        \cmidrule(lr){3-4}
        \multicolumn{2}{r}{} & $\MLNet$ & $\VNet$ \\
        \midrule
        \multirow{4}{*}{uniform} & $10$ & $2.24\textrm{e}{-4}\: {\pm}\: 9.62\textrm{e}{-5}$ & $9.75\textrm{e}{-5}\: {\pm}\: 2.10\textrm{e}{-5}$ \\
         & $50$ & $2.21\textrm{e}{-4}\: {\pm}\: 2.96\textrm{e}{-5}$ & $2.08\textrm{e}{-3}\: {\pm}\: 2.79\textrm{e}{-3}$ \\
         & $100$ & $2.27\textrm{e}{-4}\: {\pm}\: 2.72\textrm{e}{-5}$ & $8.89\textrm{e}{-5}\: {\pm}\: 1.00\textrm{e}{-5}$ \\
         & $200$ & $2.31\textrm{e}{-4}\: {\pm}\: 3.50\textrm{e}{-5}$ & $2.05\textrm{e}{-3}\: {\pm}\: 2.76\textrm{e}{-3}$ \\
        \hline
        \multirow{4}{*}{log-normal} & $10$ & $2.15\textrm{e}{-4}\: {\pm}\: 6.18\textrm{e}{-5}$ & $2.83\textrm{e}{-3}\: {\pm}\: 3.70\textrm{e}{-3}$ \\
         & $50$ & $9.73\textrm{e}{-4}\: {\pm}\: 1.04\textrm{e}{-3}$ & $3.71\textrm{e}{-3}\: {\pm}\: 4.91\textrm{e}{-3}$ \\
         & $100$ & $2.64\textrm{e}{-4}\: {\pm}\: 1.73\textrm{e}{-5}$ & $2.82\textrm{e}{-3}\: {\pm}\: 4.31\textrm{e}{-3}$ \\
         & $200$ & $3.00\textrm{e}{-4}\: {\pm}\: 1.43\textrm{e}{-5}$ & $1.97\textrm{e}{-4}\: {\pm}\: 4.22\textrm{e}{-5}$ \\
        \hline
        \multirow{2}{*}{\makecell{cookie fixed}} & $16$ & $9.41\textrm{e}{-4}\: {\pm}\: 1.12\textrm{e}{-4}$ & $1.10\textrm{e}{-3}\: {\pm}\: 3.76\textrm{e}{-4}$ \\
         & $64$ & $1.85\textrm{e}{-3}\: {\pm}\: 1.38\textrm{e}{-4}$ & $7.20\textrm{e}{-4}\: {\pm}\: 1.43\textrm{e}{-4}$ \\
        \hline
        \multirow{2}{*}{\makecell{cookie variable}} & $32$ & $4.69\textrm{e}{-3}\: {\pm}\: 1.41\textrm{e}{-3}$ & $3.69\textrm{e}{-3}\: {\pm}\: 4.53\textrm{e}{-4}$ \\
         & $128$ & $5.98\textrm{e}{-3}\: {\pm}\: 1.13\textrm{e}{-4}$ & $3.08\textrm{e}{-3}\: {\pm}\: 5.07\textrm{e}{-4}$ \\
        \bottomrule
    \end{tabular}
    \endgroup
    \caption{$\Err$ for $\MLNet$ and $\VNet$ evaluated on all test cases.}
    \label{tab:results_mrh1}
\end{table}

\begin{table}
    \centering
    \begingroup
    \renewcommand\cellalign{l}
    \setcellgapes{0.2ex}\makegapedcells
    \begin{tabular}{
    l c*{4}{>{\hspace*{0.0mm}}c<{\hspace*{0.0mm}}}}
        \toprule
        \multirow{2}[2]{*}{problem} & \multirow{2}[2]{*}{\makecell{parameter\\dimension $p$}} & \multicolumn{2}{c}{$\ErrRef$} \\
        \cmidrule(lr){3-4}
        \multicolumn{2}{r}{} & $\MLNet$ & $\VNet$ \\
        \midrule
        \multirow{4}{*}{uniform} & $10$ & $5.33\textrm{e}{-3}\: {\pm}\: 4.50\textrm{e}{-6}$ & $5.33\textrm{e}{-3}\: {\pm}\: 7.01\textrm{e}{-7}$ \\
         & $50$ & $5.33\textrm{e}{-3}\: {\pm}\: 1.15\textrm{e}{-6}$ & $6.24\textrm{e}{-3}\: {\pm}\: 1.28\textrm{e}{-3}$ \\
         & $100$ & $5.33\textrm{e}{-3}\: {\pm}\: 1.63\textrm{e}{-6}$ & $5.33\textrm{e}{-3}\: {\pm}\: 8.55\textrm{e}{-7}$ \\
         & $200$ & $5.33\textrm{e}{-3}\: {\pm}\: 1.46\textrm{e}{-6}$ & $6.22\textrm{e}{-3}\: {\pm}\: 1.25\textrm{e}{-3}$ \\
        \hline
        \multirow{4}{*}{log-normal} & $10$ & $5.33\textrm{e}{-3}\: {\pm}\: 2.93\textrm{e}{-6}$ & $6.78\textrm{e}{-3}\: {\pm}\: 2.04\textrm{e}{-3}$ \\
         & $50$ & $5.51\textrm{e}{-3}\: {\pm}\: 2.49\textrm{e}{-4}$ & $7.53\textrm{e}{-3}\: {\pm}\: 3.10\textrm{e}{-3}$ \\
         & $100$ & $5.33\textrm{e}{-3}\: {\pm}\: 5.54\textrm{e}{-7}$ & $6.90\textrm{e}{-3}\: {\pm}\: 2.71\textrm{e}{-3}$ \\
         & $200$ & $5.34\textrm{e}{-3}\: {\pm}\: 3.03\textrm{e}{-6}$ & $5.33\textrm{e}{-3}\: {\pm}\: 1.69\textrm{e}{-6}$ \\
        \hline
        \multirow{2}{*}{\makecell{cookie fixed}} & $16$ & $7.09\textrm{e}{-2}\: {\pm}\: 1.87\textrm{e}{-5}$ & $7.09\textrm{e}{-2}\: {\pm}\: 7.39\textrm{e}{-6}$ \\
         & $64$ & $9.73\textrm{e}{-2}\: {\pm}\: 1.16\textrm{e}{-5}$ & $9.73\textrm{e}{-2}\: {\pm}\: 5.62\textrm{e}{-6}$ \\
        \hline
        \multirow{2}{*}{\makecell{cookie variable}} & $32$ & $7.83\textrm{e}{-2}\: {\pm}\: 3.22\textrm{e}{-4}$ & $7.81\textrm{e}{-2}\: {\pm}\: 2.39\textrm{e}{-4}$ \\
         & $128$ & $1.12\textrm{e}{-1}\: {\pm}\: 2.58\textrm{e}{-4}$ & $1.12\textrm{e}{-1}\: {\pm}\: 2.76\textrm{e}{-5}$ \\
        \bottomrule
    \end{tabular}
    \endgroup
    \caption{$\ErrRef$ for $\MLNet$ and $\VNet$ evaluated on all test cases.}
    \label{tab:results_ref_mrh1}
\end{table}

\subsection{Training with successively fewer samples on finer levels}
\label{sec:sample distribution}

A striking advantage of stochastic multilevel methods is that the majority of sample points can be computed on coarser levels with low effort while only few samples are needed on the finest grid \citep{teckentrup2015,lye_mishra_molinaro_2021,harbrecht2016multilevel,ballani2016multilevel}.
We transfer this concept to $\MLNet$ by exponentially decreasing the number of training samples from level to level. To this end, we train each level of our $\MLNet$ only with a fraction of the dataset. More concretely, we divide the number of samples in each subsequent level by two, i.e for level $\ell=1,\ldots,L$, we use $N_\ell\coloneqq 2^{1-\ell} \times 10^{4}$ training samples.
We observe that alternating between low level samples and samples for which all corrections are known during training is needed for stable optimization. The computational budget of generating a multiscale dataset with 1000 samples on the coarsest level and exponentially decaying number of samples on subsequent levels corresponds to 232 full resolution samples.

Table~\ref{tab:datadecay_table} depicts the errors for two test cases (uniform and log-normal with $p=100$) for $\MLNet$ trained on the decayed multiscale dataset and $\VNet$ trained on the full resolution dataset generated with a comparable compute budget (232 samples). We make two observations: First, training $\MLNet$ with a decaying number of samples per level hardly decreases its performance when compared to the full dataset of 1000 samples from Table~\ref{tab:results_mrl2}. 
Second, $\VNet$ trained on a full-resolution dataset of comparable compute budget significantly reduces performance compared to training $\VNet$ on 1000 samples (Table~\ref{tab:results_mrl2}) and compared to $\MLNet$ on the decayed dataset (Table~\ref{tab:datadecay_table}).

The MR$L^2$ errors depicted in Table~\ref{tab:datadecay_table} illustrate the low error that can be achieved, which is at least one order of magnitude lower than what is reported in other papers, see also Tables~\ref{tab:results_mrl2} and~\ref{tab:results_ref_mrl2} in Appendix \ref{sec:add_tables}.
Note that the observed MR$H^1$ error in our experiments is bounded from below by the FE approximation, i.e. the resolution of the finest mesh. For the multilevel advantage to fully take effect, training on finer grids would be required to decrease the FE approximation errors to be comparable to the smaller NN approximation errors. Hence, using fewer high fidelity training samples is a crucial step to train models which achieve an overall MR$H^1$ accuracy comparable to traditional FE solvers.

We conclude that $\MLNet$ can effectively be trained with fewer samples for finer levels, significantly reducing the overall training and data generation costs and, thus, enabling the application of our model with much finer grids.
\begin{table}
    \centering
    \renewcommand\cellalign{l}
    \setcellgapes{0.2ex}\makegapedcells
    \begin{tabular}{l c c c c}
        \toprule
        \makecell{method} & \makecell[c]{error} & \makecell[c]{dataset} & uniform & log-normal\\
        \midrule
        $\MLNet$ & $\Err$ & \text{decaying} & $2.86\textrm{e}{-4}\: {\pm}\: 2.24\textrm{e}{-5}$ & $4.05\textrm{e}{-4}\: {\pm}\: 2.48\textrm{e}{-5}$\\
        $\VNet$ & $\Err$ & \text{fixed} & $1.14\textrm{e}{-3}\: {\pm}\: 6.47\textrm{e}{-4}$ & $5.92\textrm{e}{-3}\: {\pm}\: 4.66\textrm{e}{-3}$\\
        \midrule
        $\MLNet$ & $\ErrRef$ & \text{decaying} & $5.34\textrm{e}{-3}\: {\pm}\: 1.22\textrm{e}{-6}$ & $5.34\textrm{e}{-3}\: {\pm}\: 1.61\textrm{e}{-6}$\\
        $\VNet$ & $\ErrRef$ & \text{fixed} & $5.49\textrm{e}{-3}\: {\pm}\: 1.59\textrm{e}{-4}$ & $8.53\textrm{e}{-3}\: {\pm}\: 3.53\textrm{e}{-3}$\\
        \midrule
        $\MLNet$ & $\Errl$ & \text{decaying} & $5.75\textrm{e}{-5}\: {\pm}\: 4.43\textrm{e}{-6}$ & $9.39\textrm{e}{-5}\: {\pm}\: 5.42\textrm{e}{-6}$\\
        $\VNet$ & $\Errl$ & \text{fixed} & $5.38\textrm{e}{-4}\: {\pm}\: 3.93\textrm{e}{-4}$ & $3.76\textrm{e}{-3}\: {\pm}\: 3.23\textrm{e}{-3}$\\
        \midrule
        $\MLNet$ & $\ErrRefl$ & \text{decaying} & $6.65\textrm{e}{-5}\: {\pm}\: 2.16\textrm{e}{-6}$ & $9.91\textrm{e}{-5}\: {\pm}\: 4.59\textrm{e}{-6}$\\
        $\VNet$ & $\ErrRefl$ & \text{fixed} & $5.40\textrm{e}{-4}\: {\pm}\: 3.91\textrm{e}{-4}$ & $3.76\textrm{e}{-3}\: {\pm}\: 3.23\textrm{e}{-3}$\\
        \bottomrule
    \end{tabular}
    \caption{All errors for $\MLNet$ and $\VNet$ trained with a reduced dataset size on two problem cases with parameter dimension $p=100$. The $\MLNet$ is trained with the number of training samples halved for each level and $\VNet$ is trained using $232$ samples in total.}
    \label{tab:datadecay_table}
\end{table}

\section{Conclusion}
\label{sec:conclusion}

In this work, we combine concepts from established multilevel algorithms with NNs for efficiently solving challenging high-dimensional parametric PDEs. We provide a theoretical complexity analysis revealing the relation between UNet-like architectures approximate classical multigrid algorithms arbitrarily well. Moreover, we show that the $\MLNet$ architecture presents an advantageous approach for learning the parameter-to-solution operator of the parametric Darcy problems. The performance of our method is illustrated by several numerical benchmark experiments showing a significant improvement over the state-of-the-art of previous DL-based approaches for parametric PDEs. In fact, the shown approximation quality matches the best-in-class techniques such as low-rank least squares methods, stochastic Galerkin approaches, compressed sensing and stochastic collocation.
Additionally, we show that the multilevel architecture allows to use fewer data points for finer corrections during training with negligible impact on practical performance. This enables the extension to much finer resolutions without a prohibitive increase of computational complexity. Finally, we note that while we consider a scalar linear elliptic equation in this work, our methodology can be used for a variety of possibly nonlinear and vector-valued problems. Future research directions could include the analysis of such problems as well as the extension of our method to adaptive grids.

\acks{
The authors would like to thank Reinhold Schneider for fruitful discussions on the topic.
I.\ Gühring acknowledges support from the Research Training Group ``Differential Equation- and Data-driven Models in Life Sciences and Fluid Dynamics: An Interdisciplinary
Research Training Group (DAEDALUS)'' (GRK 2433) funded by the German Research Foundation (DFG) and a post-doctoral scholarship from Technical University of Berlin for ``Deep Learning for (Parametric) PDEs from a Theoretical and Practical Perspective''.
M.\ Eigel acknowledges partial support from DFG SPP 1886 ``Polymorphic uncertainty modelling for the numerical design of structures'' and DFG SPP 2298 ``Theoretical foundations of Deep Learning''.}
We acknowledge the computing facilities and thank the IT support of the Weierstrass Institute.

\vskip 0.2in
\bibliography{Bibliography}

\appendix
\section{Exact \texorpdfstring{$\MLNet$}{ML-Net} architecture and training details}\label{app:implementation}
 In our experiments, we construct $\MLNet$ and $\VNet$ with $L=7$ levels. Convolutional layers use $3{\times}3$ kernels with zero padding. Upsampling ($\uparrow$) and downsampling ($\downarrow$) layers are implemented by $5{\times}5$ transpose-strided and strided convolutions with 2-strides, respectively. We do not employ any normalization layers and use the ReLU activation function throughout all architectures. Apart from up- and downsampling layers, $\UNet$ subunits of both $\MLNet$ and $\VNet$ contain two convolutional layers on each scale. Skip-connections are realized by concatenating the output of preceding layers.

For $\MLNet$, we choose the number of $\UNet$s per level $\mathbf{R}_\ell$, $\ell=1,\ldots,L$, such that each subnetwork $\MLNet_\ell$ has approximately $8{\times} 10^5$ trainable parameters. This number was empirically found to produce satisfactory results and yields $\mathbf{R}=[11, 9, 5, 4, 3, 2, 2]$. To downsample the diffusivity field $\bfkappa$ to the input resolution required by the subnetwork $\MLNet_\ell$, $L - \ell + 1$ strided convolutions followed by ReLU activations are used. We use $32$ channels for convolutional layers of $\MLNet$, except for the coarsest level where $64$ channels are used. For all levels $\ell=2,\ldots,L$ except the coarsest one, the $\MLNet_\ell$ subnetworks output the predictions as four channels at half the resolution and use pixel un-shuffle layers \citep{pixelshuffle} to assemble the full predictions. The reason for this is that fine grid corrections for a dyadic subdivision yield a higher similarity between every second value than neighbouring values in the array.
For $\VNet$, $64$ channels are used for all convolutional layers.

$\MLNet$ and $\VNet$ are trained for $200$ epochs with an initial learning rate of $10^{-3}$ during the first $60$ epochs. The learning rate is then linearly decayed to $2{\times}10^{-5}$ over the next $100$ epochs, where it was held for the rest of the training. Due to memory constraints, the batch sizes were chosen to be $20$ for $\MLNet$ and $16$ for $\VNet$. The Adam optimizer \citep{adam_paper} was used in the standard configuration with parameters $\beta_1=0.99, \beta_2=0.999$, and without weight decay.

\section{Additional tables} \label{sec:add_tables}

The following tables depict relative $L^2$ errors of the considered NN architectures for all presented test cases.

\begin{table}[H]
    \centering
    \begingroup
    \renewcommand\cellalign{l}
    \setcellgapes{0.2ex}\makegapedcells
    \begin{tabular}{
    l c*{4}{>{\hspace*{0.0mm}}c<{\hspace*{0.0mm}}}}
        \toprule
        \multirow{2}[2]{*}{problem case} & \multirow{2}[2]{*}{\makecell{parameter\\dimension $p$}} & \multicolumn{2}{c}{$\Errl$} \\
        \cmidrule(lr){3-4}
        \multicolumn{2}{r}{} & $\MLNet$ & $\VNet$ \\
        \midrule
        \multirow{4}{*}{uniform} & $10$ & $9.28\textrm{e}{-5}\: {\pm}\: 6.92\textrm{e}{-5}$ & $3.72\textrm{e}{-5}\: {\pm}\: 9.18\textrm{e}{-6}$ \\
          & $50$ & $4.88\textrm{e}{-5}\: {\pm}\: 6.82\textrm{e}{-6}$ & $1.21\textrm{e}{-3}\: {\pm}\: 1.66\textrm{e}{-3}$ \\
          & $100$ & $4.90\textrm{e}{-5}\: {\pm}\: 9.56\textrm{e}{-6}$ & $3.21\textrm{e}{-5}\: {\pm}\: 3.06\textrm{e}{-6}$ \\
          & $200$ & $4.81\textrm{e}{-5}\: {\pm}\: 7.38\textrm{e}{-6}$ & $1.17\textrm{e}{-3}\: {\pm}\: 1.61\textrm{e}{-3}$ \\
         \hline
         \multirow{4}{*}{log-normal} & $10$ & $7.44\textrm{e}{-5}\: {\pm}\: 1.75\textrm{e}{-5}$ & $1.66\textrm{e}{-3}\: {\pm}\: 2.22\textrm{e}{-3}$ \\
         & $50$ & $7.46\textrm{e}{-4}\: {\pm}\: 9.61\textrm{e}{-4}$ & $2.30\textrm{e}{-3}\: {\pm}\: 3.13\textrm{e}{-3}$ \\
         & $100$ & $6.54\textrm{e}{-5}\: {\pm}\: 6.71\textrm{e}{-6}$ & $1.72\textrm{e}{-3}\: {\pm}\: 2.80\textrm{e}{-3}$ \\
         & $200$ & $7.35\textrm{e}{-5}\: {\pm}\: 2.82\textrm{e}{-6}$ & $7.94\textrm{e}{-5}\: {\pm}\: 2.14\textrm{e}{-5}$ \\
        \hline
         \multirow{2}{*}{\makecell{cookie fixed}} & $16$ & $3.29\textrm{e}{-4}\: {\pm}\: 3.26\textrm{e}{-5}$ & $2.35\textrm{e}{-4}\: {\pm}\: 1.37\textrm{e}{-5}$ \\
         & $64$ & $5.32\textrm{e}{-4}\: {\pm}\: 1.80\textrm{e}{-5}$ & $1.40\textrm{e}{-4}\: {\pm}\: 4.32\textrm{e}{-5}$ \\
        \hline
         \multirow{2}{*}{\makecell{cookie variable}} & $32$ & $1.33\textrm{e}{-3}\: {\pm}\: 1.05\textrm{e}{-4}$ & $7.68\textrm{e}{-4}\: {\pm}\: 2.45\textrm{e}{-5}$ \\
         & $128$ & $2.01\textrm{e}{-3}\: {\pm}\: 8.14\textrm{e}{-5}$ & $7.14\textrm{e}{-4}\: {\pm}\: 7.74\textrm{e}{-5}$ \\
        \bottomrule
    \end{tabular}
    \endgroup
    \caption{$\Errl$ for $\MLNet$ and $\VNet$ evaluated on all test cases.}
    \label{tab:results_mrl2}
\end{table}

\begin{table}[H]
    \centering
    \begingroup
    \renewcommand\cellalign{l}
    \setcellgapes{0.2ex}\makegapedcells
    \begin{tabular}{
    l c*{4}{>{\hspace*{0.0mm}}c<{\hspace*{0.0mm}}}}
        \toprule
        \multirow{2}[2]{*}{problem case} & \multirow{2}[2]{*}{\makecell{parameter\\dimension $p$}} & \multicolumn{2}{c}{$\ErrRefl$} \\
        \cmidrule(lr){3-4}
        \multicolumn{2}{r}{} & $\MLNet$ & $\VNet$ \\
        \midrule
        \multirow{4}{*}{uniform} & $10$ & $1.05\textrm{e}{-4}\: {\pm}\: 6.75\textrm{e}{-5}$ & $4.77\textrm{e}{-5}\: {\pm}\: 6.80\textrm{e}{-6}$ \\
         & $50$ & $5.86\textrm{e}{-5}\: {\pm}\: 7.02\textrm{e}{-6}$ & $1.22\textrm{e}{-3}\: {\pm}\: 1.66\textrm{e}{-3}$ \\
         & $100$ & $5.50\textrm{e}{-5}\: {\pm}\: 5.07\textrm{e}{-6}$ & $4.19\textrm{e}{-5}\: {\pm}\: 3.41\textrm{e}{-6}$ \\
         & $200$ & $5.67\textrm{e}{-5}\: {\pm}\: 5.78\textrm{e}{-6}$ & $1.18\textrm{e}{-3}\: {\pm}\: 1.60\textrm{e}{-3}$ \\
        \hline
        \multirow{4}{*}{log-normal} & $10$ & $7.76\textrm{e}{-5}\: {\pm}\: 1.29\textrm{e}{-5}$ & $1.66\textrm{e}{-3}\: {\pm}\: 2.22\textrm{e}{-3}$ \\
         & $50$ & $7.59\textrm{e}{-4}\: {\pm}\: 9.68\textrm{e}{-4}$ & $2.30\textrm{e}{-3}\: {\pm}\: 3.13\textrm{e}{-3}$ \\
         & $100$ & $7.06\textrm{e}{-5}\: {\pm}\: 4.38\textrm{e}{-6}$ & $1.72\textrm{e}{-3}\: {\pm}\: 2.80\textrm{e}{-3}$ \\
         & $200$ & $8.31\textrm{e}{-5}\: {\pm}\: 3.91\textrm{e}{-6}$ & $8.69\textrm{e}{-5}\: {\pm}\: 1.97\textrm{e}{-5}$ \\
        \hline
        \multirow{2}{*}{\makecell{cookie fixed}} & $16$ & $6.19\textrm{e}{-3}\: {\pm}\: 6.60\textrm{e}{-5}$ & $6.05\textrm{e}{-3}\: {\pm}\: 3.50\textrm{e}{-5}$ \\
         & $64$ & $9.63\textrm{e}{-3}\: {\pm}\: 4.33\textrm{e}{-5}$ & $9.48\textrm{e}{-3}\: {\pm}\: 2.79\textrm{e}{-5}$ \\
        \hline
        \multirow{2}{*}{\makecell{cookie variable}} & $32$ & $8.81\textrm{e}{-3}\: {\pm}\: 3.41\textrm{e}{-5}$ & $8.49\textrm{e}{-3}\: {\pm}\: 9.10\textrm{e}{-5}$ \\
         & $128$ & $1.64\textrm{e}{-2}\: {\pm}\: 5.36\textrm{e}{-5}$ & $1.62\textrm{e}{-2}\: {\pm}\: 6.70\textrm{e}{-5}$ \\
        \bottomrule
    \end{tabular}
    \endgroup
    \caption{$\ErrRefl$ for $\MLNet$ and $\VNet$ evaluated on all test cases.}
    \label{tab:results_ref_mrl2}
\end{table}

\section{Proofs for the results of Section~\ref{sec:analysis}}\label{app:proofs}
This section is devoted to the proofs of our theoretical analysis in Section~\ref{sec:analysis}. For this, we start with some mathematical notation.

\subsection{Mathematical notation}
For some function $f \colon D \subset \RR^d \to \RR^n$, we denote
$\norm{f}_{L^\infty(D)} := \esssup_{x \in D} \norm{f(x)}_\infty$, where $\norm{\cdot}_\infty$ denotes the maximum norm of a vector. For two tensors $\bfx, \bfy \in \RR^{W\times H}$ with $W, H\in \NN$, we denote their pointwise multiplication by $\bfx \odot \bfy \in \RR^{W\times H}$. Since we consider the FE spaces $V_h$ in the two-dimensional setting on uniformly refined square meshes, FE coefficient vectors $\bfx \in \RR^{\dim V_h}$ can be viewed as two-dimensional arrays. Whenever $x$ is processed by a CNN, we implicitly assume a 2D matrix representation.

The bilinear form in~\eqref{eq:variational darcy} induces the problem related and parameter dependent \emph{energy norm} given by
\begin{equation}
\label{eq:energy_norm}
    \norm{w}_{A_\bfy}^2 := \int_D \kappa(\bfy,x) |\nabla w(x)|^2 \dx x \quad\text{for}\quad w\in V.
\end{equation}
Consider Problem~\ref{def:darcy_discretized}. Under the uniform ellipticity assumption\footnote{which always is satisfied with high probability in our settings}, the energy norm is equivalent to the $H^1$~norm (see e.g.~\cite{cohen2015approximation_long}).
We denote by $c_{H^1}, C_{H^1} > 0$ the grid independent constants such that for all $u \in V$ and $\bfy \in \Gamma$
    \begin{equation}
    \label{eq:norm_inequality}
        c_{H^1} \norm{u}_{A_\bfy} \leq \hnorm{u} \leq C_{H^1} \norm{u}_{A_\bfy}.
    \end{equation}
    Moreover, for a multilevel decomposition up to level $L\in \NN$ and $\hat{v}_1,\ldots,\hat{v}_L$ defined as in~\eqref{eq:correction_solution_operator}, denote by $C_{\textnormal{corr}} > 0$ the constant such that for all $\bfy \in \Gamma$ and $\ell=1,\ldots,L$
    \begin{equation}
    \label{eq:v_hat_bound}
        \hnorm{\hat{v}_\ell(\bfy)} \leq C_{\textnormal{corr}} 2^{-\ell} \norm{f}_\ast.
    \end{equation}
Note that for any $\bfy \in \Gamma$, the solution $v_h(\bfy)$ of Problem~\ref{def:darcy_discretized} satisfies
\begin{equation}
    \norm{v_h(\bfy)}_{A_\bfy} \leq C_{H^1} \norm{f}_\ast.
\end{equation}

To limit excessive mathematical overhead in our proofs, we restrict ourselves to a certain type of uniform square grid specified in the following remark.
\begin{remark}
\label{rmk:uniform_square_mesh}
    For $D=[0,1]^2$, mesh width $h>0$ and $m := \nicefrac{1}{h} \in \NN$, we exclusively consider the uniform square grid comprised of $m^2$ identical squares each subdivided into two triangles as illustrated in Figure~\ref{fig:surrounding_triangles}.
\end{remark}

\subsection{CNNs: terminology and basic properties}
\label{sec:cnns_terminology_and_basic_properties}
In this section, we briefly introduce basic convolutional operations together with their corresponding notation. We mostly focus on the two-dimensional case. However, all definitions and results can easily be extended to arbitrary dimensions. For a more in-depth treatment of the underlying concepts, we refer to~\citep[Chapter~9]{goodfellow2016deep}.

Convolutions used in CNNs deviate in some details from the established mathematical definition of a convolution. In the two-dimensional case, the input consists of a tensor $\bfx_{\text{in}} \in \RR^{\Cin \times \Width_{\text{in}} \times \Height_{\text{in}}}$ with spatial dimensions $\Width_{\text{in}}, \Height_{\text{in}} \in \NN$ and number of channels $\Cin \in \NN$, and the output is given by a tensor $\bfx_{\text{out}} \in \RR^{\Cout \times \Width_{\text{out}} \times \Height_{\text{out}}}$. Here, $\Cout\in \NN$ denotes the number of output channels and $\Width_{\text{out}}, \Height_{\text{out}} \in \NN$ the potentially transformed spatial dimensions. $\bfx_{\text{out}}$ is the result of convolving $\bfx_\text{in}$ with a learnable \emph{convolutional kernel} $\kernel \in \RR^{\Cin \times \Cout \times \Wk \times \Wk}$ with kernel width $\Wk \in \NN$, and the channel-wise addition of a learnable \emph{bias} $B \in \RR^{\Cout}$. We use three different types of convolutions: (i) \emph{vanilla}, denoted by $\bfx_{\text{in}} \star K$, here $\Wout = \Win$ and $\Hout= \Hin$; (ii) \emph{two-strided} denoted by $\bfx_{\text{in}} \convs K$, here $\Wout = \left\lfloor\Win / 2\right\rfloor - 1$ and $\Hout= \left\lfloor \Hin / 2\right\rfloor - 1$; (iii) \emph{two-transpose-strided}, denoted by $\bfx_{\text{in}} \convts K$, here $\Wout \coloneqq 2\Win + 1$ and $\Hout\coloneqq 2\Hin + 1$.
Moreover, we write $M(\Phi) \in \NN$ for the number of trainable parameters and $L(\Phi) \in \NN$ for the number of layers of a (possibly convolutional) NN $\Phi$.

Many previous works have focused on approximating functions by fully connected NNs. With a simple trick, we transfer these approximation results to CNNs. 
\begin{theorem}\label{thm:FCNN_to_CNN}
Let $D\subset \RR^d$ and $f:D \to  \RR$ be some function. Furthermore, let $\veps > 0$ and $\Phi$ be a fully connected NN with $d$-dimensional input and one-dimensional output such that $\norm{f - \Phi}_{L^\infty(D)} \leq \veps$, then there exists a CNN $\Psi$ with 
\begin{enumerate}[label=(\roman*)]
    \item $d$ input channels and one output channel;
    \item the spatial dimension of all convolutional kernels is $1\times 1$;
    \item  the same activation function as $\Phi$;
    \item  $M(\Phi) = M(\Psi)$ and $L(\Phi) = L(\Psi)$;
    \item for all $\Width \in \NN$, we have $$
\norm{\Psi - \hat f}_{L^\infty(D^{\Width \times \Width})}\leq \veps,
$$
where $\hat f: D^{\Width \times \Width} \to \RR^{\Width \times \Width}$ is the component-wise application of $f$.

\end{enumerate} 
\end{theorem}
\begin{proof}
The proof follows directly by using the affine linear transformations in each layer of $\Phi$ in the channel dimension as $1\times 1$ convolutions.
\end{proof}

In the next corollary, Theorem~\ref{thm:FCNN_to_CNN} is used to approximate the pointwise multiplication function of two input tensors by a CNN.

\begin{corollary}\label{cor:CNN_multiplication}
Let $\varrho\in L^{\infty}_{\mathrm{loc}}(\RR)$ such that there exists $x_0\in \RR$ with $\varrho$ is three times continuously differentiable in a neighborhood of some $x_0\in\RR$ and $\varrho''(x_0)\neq 0.$  	
	Let $W\in\NN$, $B> 0$, and $\epsilon \in \epsin$, then there exists a CNN $\apmult$ with activation function $\varrho$, a two-channel input and one-channel output that satisfies the following properties:
	\begin{enumerate}
		\item \label{item:network_approximation}$\norm{\apmult(\mathbf{x},\mathbf{y}) - \mathbf{x}\odot\mathbf{y}}_{L^\infty([-B,B]^{2\times W\times W}, d\bfx d\bfy)}\leq \epsilon$;
		\item \label{item:network_complexity_apmult} $L(\apmult)=2$ and $M(\apmult)\leq 9$;
		\item the spatial dimension of all convolutional kernels is $1\times 1$.
	\end{enumerate}
\end{corollary}
\begin{proof}
The proof follows from Theorem~\ref{thm:FCNN_to_CNN} together with \cite[Corollary C.3]{guhring2021approx}.
\end{proof}

\subsection{Approximating isolated V-cycle building blocks}
One of the main intermediate steps to approximate the full multigrid cycle by CNNs is the approximation of the operator $A_\kappa$. The theoretical backbone of this section is the observation that $\ATau_{\kappa} \bfu$ can be approximated by a CNN acting on $\kappa$ (in an integral representation defined in Definition~\ref{def:ATau}) and $\bfu$. We start by defining some basic concepts used for our proofs. For the rest of this section, we set $D = [0, 1]^2$.
\begin{definition}
\label{def:ATau}
    Let $\mathcal{T}$ be a uniform triangulation of $D$ with corresponding conforming P1 FE space $V_h$ with basis $\{\phi_1, \ldots, \phi_{\dim V_h}\}$. Furthermore, let $\kappa \in H_0^1(D)$.
    \begin{enumerate}[label=(\roman*)]
         \item For $i\in \{1,\ldots, \dim V_h\}$ we set $\neigh2d(i)\coloneqq \{j\in \{1,\ldots, \dim V_h\}:\supp \phi_i \cap \supp \phi_j \neq \emptyset \}$.
        \item \label{it:def_tau_triang} For each vertex $i\in \{1,\ldots, \dim V_h\}$ of the triangulation, we enumerate the six adjacent triangles (arbitrarily but in the same order for every $i$) and denote them by $T_i^{(k)}$ with $k=1,\ldots, 6$ (see Figure~\ref{fig:surrounding_triangles} for an illustration).
        \item For $i =1, \ldots, \dim V_h$ and $k=1,\ldots, 6$, we set
        \begin{align*}
        \tmeans(\kappa, \mathcal{T}, k, i) := \int_{T_i^{k}} \kappa \dx x
        \end{align*}
        and use the notation 
        \[
        \tmeans(\kappa, \mathcal{T}, k) \coloneqq [\tmeans(\kappa, \mathcal{T}, k, i)]_{i=1}^{\dim V_h} \in \RR^{\dim V_h} 
        \]
        and 
        \[\tmeans(\kappa, \mathcal{T}) \coloneqq [\tmeans(\kappa, \mathcal{T}, k)]_{k=1}^{6}\in \RR^{6\times\dim V_h}.
        \]
    \end{enumerate}
    
\end{definition}

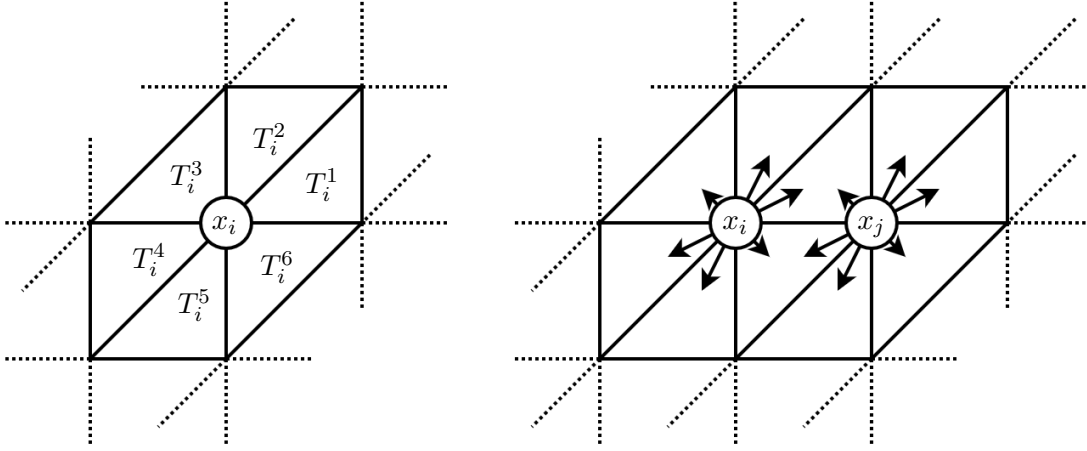
\begin{figure}
    \centering
    \vspace{-1em}
    \begin{tikzpicture}
    \node at (0, 0) {\includegraphics[scale=0.32]{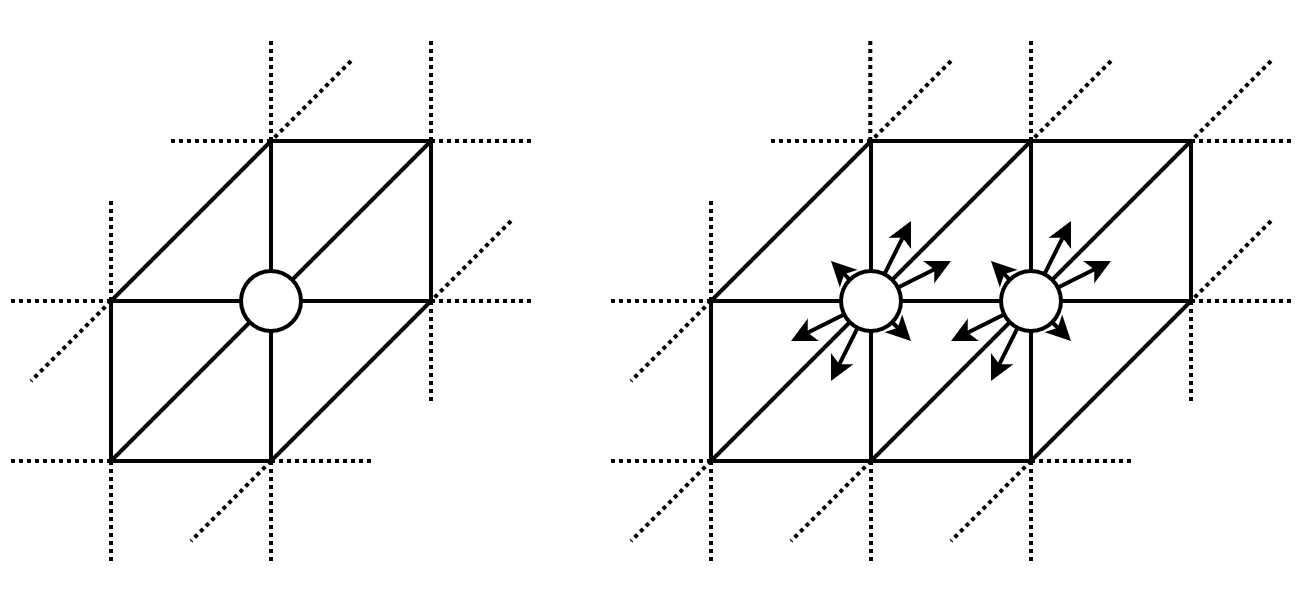}};
    \node at (-4.3, -0.04) {$x_i$};
    \node at (-3, 0.5) {$T_i^{1}$};
    \node at (-3.7, 1.1) {$T_i^{2}$};
    \node at (-4.8, 0.6) {$T_i^{3}$};
    \node at (-5.3, -0.5) {$T_i^{4}$};
    \node at (-4.7, -1.1) {$T_i^{5}$};
    \node at (-3.6, -0.6) {$T_i^{6}$};
    \node at (2.49, -0.04) {$x_i$};
    \node at (4.3, -0.06) {$x_j$};
    \end{tikzpicture}
    \caption{Illustration showing a possible enumeration of the six adjacent triangles of a grid point on the uniform square mesh from Definition~\ref{def:ATau}~\ref{it:def_tau_triang}. Note that this assignment is redundant, as on the right, one can see that $T_i^{(6)} = T_j^{(4)}$ and $T_i^{(1)} = T_j^{(3)}$.}
    \label{fig:surrounding_triangles}
\end{figure}

In the next lemma, we show that the integrals $\tmeans(\kappa, \mathcal{T}, k, i)$ can be computed via a convolution from (a discretized) $\kappa$ and, furthermore, that integrals over a coarse grid can be computed from fine-grid-integrals, again via a convolution.
\begin{lemma}
\label{lmm:kappa_conv}
Let $\Tri_h, \Tri_{2h}$ be nested triangulations of fineness $h$ and $2h$, respectively, with corresponding FE spaces $V_{2h}\subset V_h\subset H_0^1(D)$. Then, the following holds:
\begin{enumerate}[label=(\roman*)]
    \item\label{item:lmm_kappa_conv_i} There exists a convolutional kernel $K \in \RR^{1 \times 6 \times 3 \times 3}$ such that for every $\kappa \in V_h$, we have for $k \in \{1,\ldots,6\}$
    \begin{align*}
     (\bfkappa \convv K)[k] = \tmeans(\kappa, \Tri_h, k).
    \end{align*}
    \item\label{item:lmm_kappa_conv_ii} There exists a convolutional kernel $K \in \RR^{6 \times 6 \times 3 \times 3}$ such that for every $\kappa \in H_0^1(D)$ and $k=1,\ldots,6$, we have
    \begin{align*}
      ([\tmeans(\kappa, \Tri_h, 1),\ldots,\tmeans(\kappa, \Tri_h, 6)] \convv K) [k] = \tmeans(\kappa, \Tri_{2h}, k).
    \end{align*}
\end{enumerate}
\end{lemma}

\begin{proof}
    We start by showing (i). For $k=1,\ldots 6$ and $i=1, \ldots, \dim V_h$,
    \begin{equation*}
        \tmeans(\kappa, \Tri, k, i) = \int_{T_i^{k}} \kappa \dx x
        = \sum_{j=1}^{\dim V_h} \bfkappa_j \int_{T_i^{k}} \phi_j^{(h)} \dx x
        = \sum_{\{j : \supp \phi_j^{(h)} \cap T_i^k \neq \emptyset\}} \bfkappa_j \frac{h^2}{3},
    \end{equation*}
    where we used in the last step that $\int_{T_i^{k}} \kappa \dx x = \nicefrac{h^2}{3}$ if $\supp \phi_j^{(h)} \cap T_i^k \neq \emptyset$.
    Clearly, $\supp \phi_j^{(h)} \cap T_i^k=\emptyset$ for all $j\in\{1,\ldots,\dim V_h\}$ with $j \notin \neigh2d(i)$. Furthermore, the position of the relevant neighbors $\{j : \supp \phi_j^{(h)} \cap T_i^k \neq \emptyset\}$ relative to index $i$ is invariant to 2D-translations of $i$. Combining these two observations concludes the proof.

    (ii) Intuitively it is clear that a strided convolution is able to locally compute the mean of all coefficients from the finer triangulation. However, to prove this rigorously would be a tedious exercise. We hence omit the formal proof and refer to Figure~\ref{fig:sub_triangles_enumeration} for an illustration of the neighboring sub-triangles in each grid point.
\end{proof}

\begin{figure}[h]
    \centering
    \begin{tikzpicture}
    \node at (0, 0) {\includegraphics[scale=0.25]{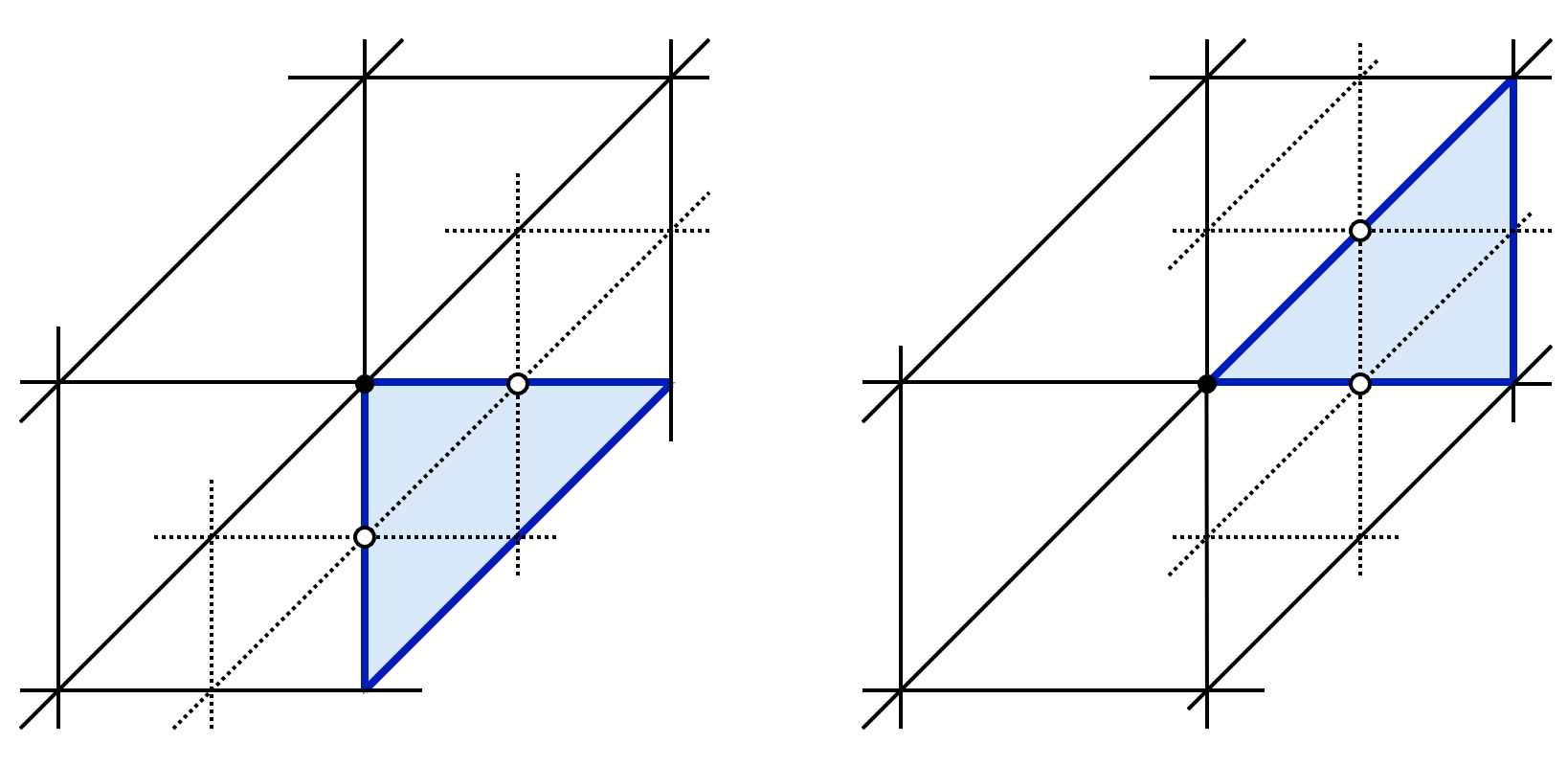}};
    \node at (-4.13, 0.2) {$x_i$};
    \node at (-2.73, 0.25) {$x_{j_1}$};
    \node at (5.7, -0.25) {$x_{j_1}$};
    \node at (5, 1.66) {$x_{j_3}$};
    \node at (-4.17, -1.15) {$x_{j_2}$};
    \node at (3.67, 0.2) {$x_i$};
    
    \node at (-2.5, -2.3) {\colortxt{blue}{$(T_{2h})_i^{(6)}$}};
    \node at (-3.35, -0.3) {\footnotesize{$(T_{h})_i^{(6)}$}};
    \node at (-2.95, -1.1) {\footnotesize{$(T_{h})_{j_2}^{(1)}$}};
    \node at (-3.35, -1.7) {\footnotesize{$(T_{h})_{j_2}^{(6)}$}};
    \node at (-1.95, -0.3) {\footnotesize{$(T_{h})_{j_1}^{(6)}$}};
    
    \node at (5.6, 2.3) {\colortxt{blue}{$(T_{2h})_i^{(1)}$}};
    \node at (6.2, 1.71) {\footnotesize{$(T_{h})_{j_3}^{(1)}$}};
    \node at (5.8, 1.11) {\footnotesize{$(T_{h})_{j_1}^{(2)}$}};
    \node at (6.2, 0.31) {\footnotesize{$(T_{h})_{j_1}^{(1)}$}};
    \node at (4.8, 0.31) {\footnotesize{$(T_{h})_{i}^{(1)}$}};
    \end{tikzpicture}
    \caption{Illustration of the neighboring triangles and sub-triangles of a finer mesh at some grid point. This figure shows how an adjacent triangle (blue) to the grid point $x_i$ in the coarse grid is subdivided. Note that all contained finer triangles can be assigned to a grid point directly adjacent to $x_i$ in the fine mesh. Hence, summing over the diffusivity values associated with the finer triangles can be represented using a convolutional ($3\times3$)-kernel in the fine mesh.}
    \label{fig:sub_triangles_enumeration}
\end{figure}
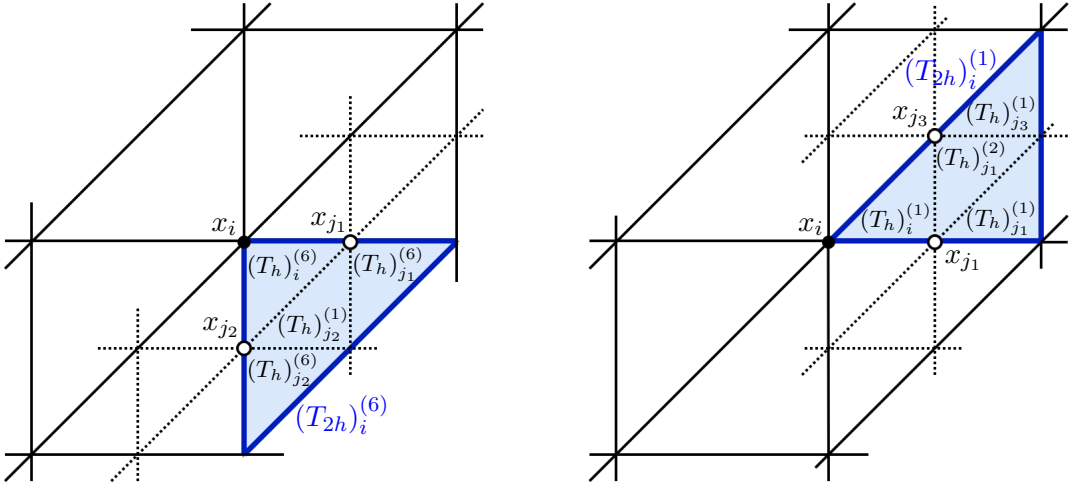

The next theorem shows that using the integral representation of $\kappa$ from Definition~\ref{def:ATau}, we can now represent $\ATau_{\kappa} \bfu$ by the pointwise multiplication of the $\kappa$ integrals with the output of a convolution applied to $\bfu$.
\begin{theorem}[Representation of $\ATau_{\kappa} \bfu$]
\label{thm:a_kappa_representation}
    For $k=1,\ldots, 6$, there exist convolutional kernels $K^{(k)}\in \RR^{1\times 1 \times 3 \times 3}$, such that for the function 
      \[
        \FAuk:\RR^{7\times \dim V} \to \RR^{\dim V}, \quad (\bfu, \obfkappa{1}, \ldots, \obfkappa{6})\mapsto \sum_{k=1}^6  \overline{\bfkappa}^{(k)} \odot (\bfu \convv K^{(k)}),
        \]
    it holds that $F$ is continuous and for any $\kappa \in H_0^1(D)$, we have
    \[
    \FAuk(\bfu, \tmeans(\kappa, \Tri, 1), \ldots, \tmeans(\kappa, \Tri, 6)) = \ATau_{\kappa}\bfu.
    \]
\end{theorem}
\begin{proof}
We start by introducing some notation: Let $i,j \in \{1, \ldots, \dim V_h\}$ and $k\in \{1,\ldots,6\}$ and set $C_{ijk} \in \RR$ as the constant that $\inner{\nabla \phi_i}{\nabla \phi_j}$ attains on $T_i^k$. Note that $C_{ijk}=0$ if $i\notin \neigh2d(j)$.
We now represent each entry of $A_\kappa$ as a sum of multiplications: For $i,j \in \{1,\ldots,\dim V_h\}$, we have
    \begin{equation*}
        (A_{\kappa})_{i j} = \int_{D} \kappa \inner{\nabla \phi_i}{\nabla \phi_j} \dx x
        = \sum_{k=1}^6 \int_{T_i^{k}} \kappa \inner{\nabla \phi_i}{\nabla \phi_j} \dx x
        = \sum_{k=1}^6 \Upsilon(\kappa, \mathcal{T}, k, i) C_{i j k},
    \end{equation*}
where we use Definition \ref{def:ATau} (ii) for the last step. For $j\in \{1,\ldots,\dim V_h\}$, we can now rewrite $(\ATau_{\kappa} \bfu)_j$ as 
    \begin{align*}
        (\ATau_{\kappa} \bfu)_j &=  \sum_{i =1}^{\dim V_h}\bfu_i \sum_{k=1}^6 \tmeans(\kappa, \Tri, k, j) C_{i j k}\\
        &= \sum_{k=1}^6 \tmeans(\kappa, \Tri, k, j) \sum_{i =1}^{\dim V_h}\bfu_i  C_{i j k}\\
        &= \sum_{k=1}^6 \tmeans(\kappa, \Tri, k, j) \sum_{i\in \textrm{Neigh2D}(j)} \bfu_i C_{ij k}.
    \end{align*}
    Here, we use in the first step the definition of $\ATau_{\kappa}$ and in the third step that $\supp \phi_i\cap \supp \phi_j\neq \emptyset$ only for neighboring indices. Since $ C_{ijk}$ only depends on the relative position of $i$ to $j$ (in the 2D sense) and is independent on $j$, the inner sum can be expressed in the vectorized form as a zero-padded convolution with a $3\times 3$ kernel (the number of neighbors) where the values depend on $k$, i.e.,
    \begin{equation*}
    \ATau_{\kappa} \bfu = \sum_{k=1}^6  \tmeans(\kappa, \Tri, k)  \odot (\bfu \convv K^{(k)}) = \FAuk(\bfu, \tmeans(\kappa, \Tri, 1), \ldots, \tmeans(\kappa, \Tri, 6)),
    \end{equation*}
    where $K^{(k)}$ are the respective kernels.
\end{proof}

In the following remark, we elaborate on the treatment of the boundary conditions.
\begin{remark}
    In our setting, we only consider basis functions on inner grid points. Alternatively, one could include the boundary basis functions and constrain their coefficients to zero to enforce homogeneous Dirichlet boundary conditions. Representing the FE operator as a zero-padded convolution on the inner grid vertices naturally realizes homogeneous Dirichlet boundary conditions by implicitly including boundary basis coefficients as zeros through padding.
\end{remark}

Using the representation of $\ATau_{\kappa} \bfu$ from Theorem~\ref{thm:a_kappa_representation} as a convolution followed by a pointwise multiplication, Corollary~\ref{cor:CNN_multiplication} can now be applied to obtain an approximation by a CNN.
\begin{theorem}[Approximation of $\ATau_{\kappa} \bfu$]
    \label{thm:a_kappa_approximation}
    Let $\Tri$ be a uniform triangulation of $D$ with corresponding conforming P1 FE space $V$. Furthermore, let the activation function $\varrho\in L^{\infty}_{\mathrm{loc}}(\RR)$ be such that there exists $x_0\in \RR$ with $\varrho$ is three times continuously differentiable in a neighborhood of some $x_0\in\RR$ and $\varrho''(x_0)\neq 0$. Then, for any $\veps>0$ and $M>0$, there exists a CNN $\Psi_{\veps, M}:\RR^{7\times\dim V_h} \to \RR^{\dim V_h}$ with size independent of $\veps$ and $M$ such that 
    \begin{align*}
        \norm{\Psi_{\veps, M} -F}_{L^\infty([-M,M]^{7\times\dim V_h})} \leq \veps.
    \end{align*}
\end{theorem}
\begin{proof}
    A three-step procedure leads to the final result:
    \begin{enumerate}[label=(\roman*)]
        \item Clearly, there exists a one-layer CNN that realizes the mapping 
        \[
        [\bfu,  \overline{\bfkappa}^{(1)}, \ldots, \overline{\bfkappa}^{(6)}]\mapsto [\bfu\convv K^{(k)},  \overline{\bfkappa}^{(k)}]_{k=1}^6.
        \]
        \item Corollary~\ref{cor:CNN_multiplication} yields the existence of a two-layer CNN $\tilde \Psi$ with at most $9$ weights and $1 \times 1$ kernels such that 
        \[
        \norm{\tilde \Psi([\bfu\star K^{(k)},  \overline{\bfkappa}^{(k)}]_{k=1}^6) -  [\overline{\bfkappa}^{(k)} \odot (\bfu \convv K^{(k)})]_{k=1}^6}_\infty \leq \frac{\veps}{6}.
        \]
        \item The channel-wise addition can trivially be constructed with a one-layer CNN with a $1 \times 1$ kernel.
    \end{enumerate}
    A concatenation of the above three CNNs concludes the proof.
\end{proof}

Recall that the prolongation and its counterpart, the weighted restriction, are essential operations in the multigrid cycle. The following remark states that these operations can naturally be expressed as convolutions.
\begin{remark}[Weighted restriction and prolongation]
\label{remark:cnn_restriction_prolongation}
    Consider a two-level decomposition using nested conforming P1 FE spaces $V_1 \subset V_2 \subset H_0^1(D)$ on uniform square meshes. Let $P$ be the corresponding prolongation matrix (Definition~\ref{def:prolongation_restriction}). Then, there exist convolutional kernels $K_1, K_2 \in \RR^{1 \times 1 \times 3 \times 3}$ such that for any FE coefficient vector
    \begin{enumerate}[label=(\roman*)]
        \item  $\bfu \in \RR^{\dim V_2}$, we have $\bfu \convs K_1= P^T \bfu$;
     \item $\bfu \in \RR^{\dim V_1}$, we have $\bfu \convts K_2 = P \bfu$.
    \end{enumerate}
\end{remark}

\subsection{Putting it all together}
In this section, we combine the approximation of individual computation steps of the multigrid algorithm by CNNs to approximate an arbitrary number of multigrid cycles by a CNN. For the composition of the individual approximations, we first provide an auxiliary lemma.
\begin{lemma}
\label{lmm:approximation_function_combination}
    Let $n\in \NN$ and $d_1,\ldots,d_{n+1} \in \NN$. For $i= 1,\ldots,n$ let $f_i \colon \RR^{d_i} \to \RR^{d_{i+1}}$ be continuous functions and define $F$ as their concatenation, i.e.,\
    \begin{align*}
        F \colon \RR^{d_1} \to \RR^{d_{n+1}},\quad F \coloneqq f_{n} \circ \ldots \circ f_1.
    \end{align*}
    Then, for every $M, \veps > 0$, there exist $M_1,\ldots,M_{n} > 0$ and $\veps_1,\ldots,\veps_{n} > 0$, such that the following holds true:
    If $\tilde{f}_i \colon \RR^{d_i} \to \RR^{d_{i+1}}$ are functions such that ${\Vert f_i - \tilde f_i \Vert}_{L^\infty([-M_i,M_i]^{d_i})} \leq \veps_i$ for $i=1,\ldots,n$, then
    \begin{align*}
        \norm{F - \tilde{f}_{n} \circ \ldots \circ \tilde{f}_{1}}_{L^\infty([-M,M]^{d_1})} \leq \veps.
    \end{align*}
\end{lemma}
\begin{proof}
    We prove the statement by induction over $n \in \NN$. For $n = 1$, the statement is clear.
    Define the function $g \colon \RR^{d_1} \to \RR^{d_{n}}$ for $n >1$ by
    \begin{align*}
        g := f_{n-1}\circ \ldots \circ f_1.
    \end{align*}
    We set $\veps_n\coloneqq \veps /2$ and $M_n \coloneqq \norm{g}_{L^\infty([-M,M]^{d_1})} + \veps < \infty$ (since $g$ is continuous). As $f_n$ is uniformly continuous on compact sets, there exists $0< \delta \leq \veps$ such that for all $z, \tilde{z} \in [-M_n, M_n]^{d_n}$ with $\norm{z - \tilde{z}}_\infty \leq \delta$ we have $\norm{f_n(z) - f_n(\tilde{z})}_\infty \leq \frac{\veps}{2}$.

    Let now $M_1,\ldots,M_{n-1}$ and $\veps_1,\ldots, \veps_{n-1}$ be given by the induction assumption (applied with $M=M$ and $\veps=\delta$) and let $\tilde{f}_i\colon \RR^{d_i} \to \RR^{d_{i+1}}$ be corresponding approximations for $i=1,\ldots, n-1$. Then (again by the induction assumption), it holds that
    \begin{align*}
        \norm{g - \tilde{f}_{n-1} \circ \ldots \circ \tilde{f}_{1}}_{L^\infty([-M,M]^{d_1})} \leq \delta.
    \end{align*}
    Setting $\tilde{g} := \tilde{f}_{n-1} \circ \ldots \circ \tilde{f}_1$, we get $\norm{\tilde{g}}_{L^\infty([-M,M]^{d_1})} \leq \norm{g}_{L^\infty([-M,M]^{d_1})} + \delta \leq M_n$.
    The proof is concluded by deriving that
    \begin{align*}
        &\norm{F - (\tilde{f}_{n} \circ \ldots \circ \tilde{f}_{1})}_{L^\infty([-M,M]^{d_1})}\\
        &\leq \norm{(f_n \circ g) - (f_n \circ \tilde{g})}_{L^\infty([-M,M]^{d_1})} + \norm{(f_n \circ \tilde{g}) - (\tilde{f}_{n} \circ \tilde{g})}_{L^\infty([-M,M]^{d_1})}\\
        &\leq \frac{\veps}{2} + \frac{\veps}{2} = \veps.
    \end{align*}
\end{proof}

With the preceding preparations, we are now ready to prove our main results. We start with the approximation of the multigrid algorithm by a CNN.
\begin{proofof}{Proof of Theorem \ref{thm:cnn_multigrid}}
\label{prf:cnn_multigrid}
    The overall proof strategy consists of decomposing the function $\mg_{k_0, k}^m:\RR^{3\times\dim V_{L}}\to \RR^{\dim V_{L}}$, where $V_L = V_h$, into a concatenation of a finite number of continuous functions that can either be implemented by a CNN or arbitrarily well approximated  by a CNN. An application of Lemma \ref{lmm:approximation_function_combination} then yields that the concatenation of these CNNs (which is again a CNN) approximates $\mg_{k_0, k}^m$. We make the following definitions:
    \begin{enumerate}[label=(\roman*)]
        \item \textbf{Integrating the diffusion coefficient.} Let $K \in \RR^{1 \times 6 \times 3 \times 3}$ be the convolutional kernel from Lemma~\ref{lmm:kappa_conv}\ref{item:lmm_kappa_conv_i}. For the function
        \begin{align*}
            f_{\text{in}} \colon \RR^{3\times\dim V_{L}} \to \RR^{8\times\dim V_{L}},\quad \begin{pmatrix}
            \bfu\\
            \bfkappa\\
            \bff
            \end{pmatrix} \mapsto
            \begin{pmatrix}
            \bfu\\
             \bfkappa \convv K\\
            \bff
            \end{pmatrix},
        \end{align*}
         it then follows from Lemma~\ref{lmm:kappa_conv}\ref{item:lmm_kappa_conv_i} that $f_{\text{in}}([\bfu, \bfkappa,\bff]) = [\bfu, \tmeans(\kappa_h, \Tri^{L}), \bff]$.
        \item \textbf{Smoothing iterations.} For $\ell=1,\ldots,L$ let $\FAuk_\ell:\RR^{7\times \dim V_\ell} \to \RR^{\dim V_\ell}$ be defined as in Theorem~\ref{thm:a_kappa_representation}. Then we define
        \begin{align*}
            f^{\ell}_{\text{sm}} \colon \RR^{8\times\dim V_{\ell}} \to \RR^{8\times\dim V_{\ell}},\quad
            \begin{pmatrix}
            \bfu\\
            \overline{\bfkappa}^{(1)}\\ 
            \vdots\\
            \overline{\bfkappa}^{(6)}\\
            \bff
            \end{pmatrix}\mapsto
            \begin{pmatrix}
            \bfu + \omega (\bff -\FAuk_\ell(\bfu, \obfkappa{1}, \ldots, \obfkappa{6}))\\
           \overline{\bfkappa}^{(1)}\\ 
            \vdots\\
            \overline{\bfkappa}^{(6)}\\
            \bff
            \end{pmatrix}.
        \end{align*}
        It follows from Theorem~\ref{thm:a_kappa_representation} that 
        \[
        f^{\ell}_{\text{sm}}([\bfu,  \tmeans(\kappa_h, \Tri^{\ell}),\bff]) = [\bfu + \omega(\bff - A_{\kappa_h}^{\ell} \bfu), \tmeans(\kappa_h, \Tri^{\ell}) ,\bff].
        \]

        \item \textbf{Restricted residual.} Similar as in the previous step, we define
        \begin{align*}
            f^{\ell}_{\text{resi}} \colon \RR^{8\times\dim V_{\ell}} &\to \RR^{9\times\dim V_{\ell}},\quad
            \begin{pmatrix}
            \bfu\\
            \overline{\bfkappa}^{(1)}\\ 
            \vdots\\
            \overline{\bfkappa}^{(6)}\\
            \bff
            \end{pmatrix}\mapsto
            \begin{pmatrix}
            \bfu\\
            \bff - \FAuk_\ell(\bfu, \obfkappa{1}, \ldots, \obfkappa{6})\\
           \overline{\bfkappa}^{(1)}\\ 
            \vdots\\
            \overline{\bfkappa}^{(6)}\\
            \bff
            \end{pmatrix},
        \end{align*}
and get from Theorem~\ref{thm:a_kappa_representation} that 
\[
  f^{\ell}_{\text{resi}}([\bfu,  \tmeans(\kappa_h, \Tri^{\ell}),\bff]) = [\bfu, \bff - A_{\kappa_h}^{\ell} \bfu, \tmeans(\kappa_h, \Tri^{\ell}) ,\bff].
\]
Let now $[K_1, \ldots, K_6] \in \RR^{6 \times 6 \times 3 \times 3}$ be the convolutional kernel from Lemma~\ref{lmm:kappa_conv}\ref{item:lmm_kappa_conv_ii} and $P_{\ell-1} \in \RR^{\dim V_{\ell} \times \dim V_{\ell-1}}$  the prolongation matrix from Definition \ref{def:prolongation_restriction} (where its transpose acts as the weighted restriction), then we define
\begin{align*}
            f^{\ell}_{\text{rest}} \colon \RR^{9\times\dim V_{\ell}} &\to \RR^{8\times\dim V_{\ell}} \times \RR^{8\times\dim V_{\ell-1}},\quad
            \begin{pmatrix}
            \bfu\\
            \bfr\\
            \overline{\bfkappa}^{(1)}\\ 
            \vdots\\
            \overline{\bfkappa}^{(6)}\\
            \bff
            \end{pmatrix}\mapsto
            \begin{bmatrix}
            \begin{pmatrix}
                \bfu\\
                \overline{\bfkappa}^{(1)}\\ 
            \vdots\\
            \overline{\bfkappa}^{(6)}\\
                \bff
                \end{pmatrix}\\
                \\
                \begin{pmatrix}
                \mathbf{0}\\
                 \overline{\bfkappa}^{(1)} \convv K_1\\ 
            \vdots\\
             \overline{\bfkappa}^{(6)} \convv K_6\\[0.5em]
                P_{\ell-1}^T \bfr
            \end{pmatrix}
            \end{bmatrix}.
        \end{align*}
        Note that $\tmeans(\kappa_h, \Tri^{\ell}, k) \star K_k =  \tmeans(\kappa_h, \Tri^{\ell-1},k)$ for $k=1,\ldots,6$.
        \item \textbf{Add coarse grid correction to fine grid.} 
        As above, let $P_{\ell-1} \in \RR^{\dim V_{\ell} \times \dim V_{\ell-1}}$ be the prolongation matrix from Definition \ref{def:prolongation_restriction}. We define the function
        \begin{align*}
            f^{\ell}_{\text{prol}} \colon \RR^{8\times\dim V_{\ell}} \times \RR^{8\times\dim V_{\ell - 1}} \to \RR^{8\times\dim V_{\ell}},\quad
            \begin{bmatrix}
            \begin{pmatrix}
                \bfu\\
                \overline{\bfkappa}^{(1)}\\ 
            \vdots\\
            \overline{\bfkappa}^{(6)}\\
                \bff
            \end{pmatrix}\\
            \\
            \begin{pmatrix}
                \bfe\\
               \vdots
            \end{pmatrix}
            \end{bmatrix}\mapsto
            \begin{pmatrix}
                \bfu + P_{\ell-1} \bfe\\
               \overline{\bfkappa}^{(1)}\\ 
            \vdots\\
            \overline{\bfkappa}^{(6)}\\
                \bff
            \end{pmatrix}.
        \end{align*}
              
        \item \textbf{Return solution.} We define
        \begin{align*}
            f_{\text{out}} \colon \RR^{8\cdot\dim V_{\ell}}  \to \RR^{\dim V_{\ell}},\quad
            \begin{pmatrix}
                \bfu\\
                \overline{\bfkappa}^{(1)}\\ 
            \vdots\\
            \overline{\bfkappa}^{(6)}\\
                \bff
            \end{pmatrix}\mapsto
            \bfu.
        \end{align*}
    \end{enumerate}
    We assemble a single V-cycle $\text{VC}_{k_0, k}^\ell$ on level $\ell=1,\ldots,L$ as follows
    \begin{align*}
        \text{VC}_{k_0, k}^\ell &:= \left(\bigcirc_{i=1}^k f^{\ell}_{sm}\right) \circ f^{\ell}_{\text{prol}} \circ (\text{Id}, \text{VC}_{k_0, k}^{{\ell-1}}) \circ f^{\ell}_{\text{rest}} \circ f^{\ell}_{\text{resi}} \circ \left(\bigcirc_{i=1}^k f^{\ell}_{sm} \right),\\
        \text{VC}_{k_0, k}^1 &:= \bigcirc_{i=1}^{k_0} f^{1}_{sm}.
    \end{align*}
    Finally, $\mg_{k_0,k}^m$ is given by:
    \begin{align*}
        \text{MG}_{k_0,k}^m = f_{\text{out}} \circ \left( \bigcirc_{i=1}^m\text{VC}_{k_0, k}^L \right) \circ f_{\text{in}}.
    \end{align*}

It is easy to see that the above defined functions are continuous and, thus, Lemma~\ref{lmm:approximation_function_combination} dictates the accuracy with which to approximate each part of the multigrid algorithm ($f^\ell_{\text{sm}}$, $f^{\ell}_{\text{prol}}$, $f^{\ell}_{\text{rest}}$, $f^{\ell}_{\text{resi}}$, $f^{L}_{\text{in}}$, $f^{L}_{\text{out}}$) on each level $\ell=1,\ldots,L$ to yield a final approximation accuracy of $\delta > 0$ for inputs bounded by $M>0$.\par

On each level $\ell=1,\ldots,L$, we derive the approximations fulfilling these requirements for $f^\ell_{\text{sm}}$ and $f^{\ell}_{\text{resi}}$ using Theorem~\ref{thm:a_kappa_approximation}, $f^{\ell}_{\text{prol}}$ and $f^{\ell}_{\text{rest}}$ as in Remark~\ref{remark:cnn_restriction_prolongation}, and $f^{\ell}_{\text{rest}}$ using Lemma~\ref{lmm:kappa_conv}. The approximation of $f^{L}_{\text{in}}$ is derived using Lemma~\ref{lmm:kappa_conv} and $f^{L}_{\text{out}}$ is trivially represented using a single convolution. Concatenating the individual CNNs yields a CNN $\Psi$ with an architecture resembling multiple concatenated UNets, such that
\begin{align*}
    \norm{\Psi - \text{MG}_{k_0,k}^m}_{L^\infty([-M, M]^{3 \times n})} \leq \delta.
\end{align*}
The norm equivalence of $\norm{\cdot}_{L^\infty([-M, M]^{3 \times n})}$ and $\hnorm{\cdot}$ in the finite dimensional setting directly implies inequality (i) with a suitable choice of $\delta$. The parameter bound (ii) follows from the construction of $\Psi$.
\end{proofof}

\begin{proofof}{Proof of Corollary~\ref{crl:conv_multigrid_repr}}
\label{prf:conv_multigrid_repr}
    Fix the damping parameter $\omega>0$ and $k \in \NN$ such that $\mu(k) < 1$ and let $m\in \NN$ be chosen as small as possible such that
    \begin{align*}
        m \geq \log\left(\frac{1}{\mu(k)}\right)^{-1} \log \left(\frac{2 C_{H^1}^2}{\veps}\right).
    \end{align*}
    Let $k_0 \in \NN$ be given such that the contraction factor of $k_0$ damped Richardson iterations on the coarsest grid is smaller than $\mu(k)$.
    Let $\Psi$ be the CNN from Theorem~\ref{thm:cnn_multigrid} setting $M=\max\left\{\sup_{\bfy \in \Gamma} \norm{\bfkappa_\bfy}_\infty, \norm{\bff}_\infty \right\}$ such that for all $\bfkappa, \bff \in [-M, M]^{n^2}$
    \begin{align*}
        \hnorm{\Psi(\mathbf{0}, \bfkappa, \bff) - \mg_{k_0,k}^m(\mathbf{0}, \bfkappa, \bff)} \leq \frac{1}{2} \veps \norm{f}_\ast.
    \end{align*}
    Using the contraction property of the multigrid algorithm from Theorem~\ref{thm:hackbusch_thm}, we get
    \begin{align*}
        \norm{\mg_{k_0,k}^m(\mathbf{0}, \bfkappa_\bfy, \bff) - \bfv_h(\bfy)}_{A_\bfy} &\leq \mu(k)^m \norm{\bfv_h(\bfy)}_{A_\bfy}\\
        &\leq \frac{1}{2} \veps C_{H^1}^{-1} \norm{f}_\ast.
    \end{align*}
    This yields
    \begin{align*}
        &\hnorm{\Psi(\mathbf{0}, \bfkappa_\bfy, \bff) - \bfv_h(\bfy)}\\
        &\leq \hnorm{\Psi(\mathbf{0}, \bfkappa_\bfy, \bff) - \mg_{k_0,k}^m(\mathbf{0}, \bfkappa_\bfy, \bff)} + \hnorm{\mg_{k_0,k}^m(\mathbf{0}, \bfkappa_\bfy, \bff) - \bfv_h(\bfy)}\\
        &\leq \frac{1}{2} \veps \norm{f}_\ast +  C_{H^1} \norm{\mg_{k_0,k}^m(\mathbf{0}, \bfkappa_\bfy, \bff) - \bfv_h(\bfy)}_{A_\bfy}\\
        &\leq \veps \norm{f}_\ast.
    \end{align*}
    Assuming $\veps \leq 1$, there is a constant $C>0$ such that $m \leq C \log\left(\frac{1}{\veps}\right)$. Together with the parameter count from Theorem~\ref{thm:cnn_multigrid}, this concludes the proof.
\end{proofof}

\begin{proofof}{Proof of Corollary~\ref{crl:fnet_repr_theorem}}
\label{prf:fnet_repr_theorem}
This proof is structured as follows. We consider a multilevel structure of V-Cycles to approximate the fine grid solution and derive error estimates for this approximation. Similarly to Corollary~\ref{crl:conv_multigrid_repr}, we show how a CNN of sufficient size is able to represent this algorithm and translate the error estimates to the network approximation.\par
To this end, we introduce some constants and fix the number of V-Cycle iterations done on each grid level:
Let $c := \max\{C_{H^1}, C_{\text{corr}}\}$. Fix $k \in \NN$ such that $\mu(k) < 1$ and set $m_1,\ldots,m_L \in \NN$ such that
\begin{align*}
    m_L &\geq \log\left(\frac{1}{\mu(k)}\right)^{-1} \left(\log\left(\frac{2 c^2 L}{\veps} \right) - L \log(2) \right),\\
    m_\ell &\geq \log\left(\frac{1}{\mu(k)}\right)^{-1} \log(2),\text{ for } \ell =1,\ldots,L-1.
\end{align*}
For any $\bfy \in \Gamma$, we denote by $A_\bfy^\ell \in \RR^{\dim V_\ell \times \dim V_\ell}$ be the discretized operator~\eqref{def:matrx_equation} and $\bff_\ell$ the right hand side from Problem~\ref{def:darcy_discretized} on level $\ell=1,\ldots,L$ such that $A_\bfy^\ell \bfv_\ell(\bfy) = \bff_\ell$. Moreover, for ease of notation, we adapt the functions introduced in the proof of Theorem~\ref{thm:cnn_multigrid} and set
\begin{enumerate}[label=(\roman*)]
    \item for each level $\ell \in 1,\ldots, L$, $\bfkappa_\bfy^\ell \in \RR^{6 \times \dim V_\ell}$ as the six channel tensor
    \begin{align*}
        \overline{\bfkappa}_\bfy^\ell := \left[\Upsilon(\kappa_\bfy, \mathcal{T}^\ell,1),\ldots,\Upsilon(\kappa_\bfy, \mathcal{T}^\ell,6)\right].
    \end{align*}

    \item the function $\widetilde{\mg}_{k_0, k}^{m_\ell} \colon \RR^{8 \times \dim V} \to \RR^{\dim V}$ as
    \begin{align*}
        \widetilde{\mg}_{k_0, k}^{m_\ell} := f_{\text{out}}^\ell \circ \left( \bigcirc_{i=1}^{m_\ell}\text{VC}_{k_0, k}^\ell\right),
    \end{align*}
    such that $\widetilde{\mg}_{k_0, k}^{m_\ell} \circ f_{\text{in}}^\ell = \bigcirc_{i=1}^{m_\ell} \mg_{k_0, k}^{m_\ell}$.
\end{enumerate}
We outline the following multilevel procedure for solving Problem~\ref{def:darcy_discretized} using multigrid with $\bfy \in \Gamma$ and $\bfkappa_\bfy \in \RR^{\dim V_L}$ given as input:
\begin{itemize}
    \item \textbf{$\bfkappa_\bfy$ subsampling.} Using $\bfkappa_\bfy$, compute $\overline{\bfkappa}_\bfy^L$ as outlined in Lemma~\ref{lmm:kappa_conv}~(i). Subsequently, compute $\overline{\bfkappa}_\bfy^\ell$ for each level $\ell=2,\ldots,L-1$ as shown in Lemma~\ref{lmm:kappa_conv}~(ii).

    \item $\boldsymbol{\ell=1}.$ Solve for $\bfv_1(\bfy)$ using $m_1$ iterations of the multigrid cycle. Let
    $$\tilde{\bfv}_1(\bfy) := \widetilde{\mg}_{k_0, k}^{m_1}(0, \overline{\bfkappa}_\bfy^1, \bff).$$
    
    \item $\boldsymbol{\ell=2,\ldots,L}.$ While $\hat{\bfv}_\ell(\bfy)$ is the true correction with respect to the Galerkin approximation $\bfv_{\ell-1}(\bfy)$, we introduce $\check{\bfv}_\ell(\bfy)$ as the correction with respect to $\tilde{\bfv}_{\ell -1}(\bfy)$. To this end, set
    \begin{align*}
        \check{\bfv}_\ell(\bfy) &:= \bfv_\ell(\bfy) - P_\ell \tilde{\bfv}_{\ell-1}(\bfy),\\
        \tilde{\bff}_\ell &:= \bff_{\ell} - A_\bfy^{\ell} P_\ell \tilde{\bfv}_{\ell-1}(\bfy) = A_\bfy^\ell \check{\bfv}_\ell(\bfy).
    \end{align*}
    Solve $A_\bfy^\ell \check{\bfv}_\ell(\bfy) = \tilde{\bff}_\ell$ using $m_\ell$ iterations of the multigrid V-cycle and set
    \begin{align*}
        \tilde{\bfv}_\ell(\bfy) :=  \widetilde{\mg}_{k_0, k}^{m_\ell}(0, \overline{\bfkappa}_\bfy^\ell, \tilde{\bff}_\ell) + P_\ell \tilde{\bfv}_{\ell-1}(\bfy).
    \end{align*}
\end{itemize}
This procedure yields the following estimates using the contraction from Theorem~\ref{thm:hackbusch_thm}.
\begin{itemize}
    \item $\boldsymbol{\ell=1}.$ Using the contraction property from Theorem~\ref{thm:hackbusch_thm} and $m_1 \geq -\log(2) / \log(\mu(k))$ yields
        \begin{align*}
        \norm{\tilde{\bfv}_1(\bfy) - \bfv_1(\bfy)}_{A^1_\bfy} \leq \frac{1}{2} c \norm{f}_\ast.
        \end{align*}
    \item $\boldsymbol{\ell=2,\ldots,L-1}.$ Using again Theorem~\ref{thm:hackbusch_thm}, the definitions for $m_\ell, \tilde{\bfv}_\ell, \check{\bfv}_\ell$ and $\hat{\bfv}_\ell = \bfv_\ell - P_\ell \bfv_{\ell - 1}$, it holds
        \begin{align*}
        \norm{\tilde{\bfv}_\ell(\bfy) - \bfv_\ell(\bfy)}_{A^\ell_\bfy} &= \norm{\widetilde{\mg}_{k_0, k}^{m_\ell}(0, \overline{\bfkappa}_\bfy^\ell, \tilde{\bff}_\ell) + P_\ell \tilde{\bfv}_{\ell-1}(\bfy) - \check{\bfv}_\ell(\bfy) - P_\ell \tilde{\bfv}_{\ell-1}(\bfy)}_{A^\ell_\bfy}\\
        &=\norm{\widetilde{\mg}_{k_0, k}^{m_\ell}(0, \overline{\bfkappa}_\bfy^\ell, \tilde{\bff}_\ell) - \check{\bfv}_\ell(\bfy)}_{A^\ell_\bfy}\\
        &\leq \frac{1}{2} \norm{\check{\bfv}_\ell(\bfy)}_{A^\ell_\bfy}\\
        &\leq \frac{1}{2} \norm{\hat{\bfv}_{\ell}(\bfy)}_{A^\ell_\bfy} + \frac{1}{2} \norm{P_\ell \tilde{\bfv}_{\ell-1}(\bfy) - P_\ell \bfv_{\ell-1}(\bfy)}_{A^\ell_\bfy}\\
        &\leq \frac{1}{2} c 2^{-\ell} \norm{f}_\ast + \frac{1}{2} \norm{ \tilde{\bfv}_{\ell-1}(\bfy) - \bfv_{\ell-1}(\bfy)}_{A^{\ell-1}_\bfy}.
    \end{align*}
    Here, the bound on $\hat{\bfv}_\ell(\bfy)$ follows from Equation~\eqref{eq:v_hat_bound}.
    \item $\boldsymbol{\ell=L}.$ By the same argument as for the case $\ell=2,\ldots,L-1$ together with the norm inequality in~\eqref{eq:norm_inequality} and the definition of $m_L$, we arrive at the following estimate for the $H^1$-distance on the finest level.
    \begin{align*}
            &\hnorm{\tilde{\bfv}_L(\bfy) - \bfv_L(\bfy)}\\
            &\leq c \norm{\tilde{\bfv}_L(\bfy) - \bfv_L(\bfy)}_{A^L_\bfy}\\
            &= c \norm{\widetilde{\mg}_{k_0, k}^{m_L}(0, \overline{\bfkappa}_\bfy^L, \tilde{\bff}_L) - \check{\bfv}_L(\bfy)}_{A^L_\bfy}\\
            &\leq \frac{1}{2 c L} 2^{L} \veps \norm{\check{\bfv}_L(\bfy)}_{A^L_\bfy}\\
            &\leq \frac{1}{c L} 2^{L} \veps \left( \frac{1}{2}\norm{\hat{\bfv}_L(\bfy)}_{A^L_\bfy} + \frac{1}{2} \norm{P_L \tilde{\bfv}_{L-1}(\bfy) - P_L \bfv_{L-1}(\bfy)}_{A^{L}_\bfy} \right)\\
            &\leq \frac{1}{c L} 2^{L} \veps \left( \frac{1}{2} c 2^{-L} \norm{f}_\ast + \frac{1}{2} \norm{\tilde{\bfv}_{L-1}(\bfy) - \bfv_{L-1}(\bfy)}_{A^{L-1}_\bfy} \right)
        \end{align*}
        Recursively, this yields
        \begin{align*}
            \hnorm{\tilde{\bfv}_L(\bfy) - \bfv_L(\bfy)} &\leq \frac{1}{cL} 2^{L} \veps \sum_{\ell=1}^L \left(\prod_{n=\ell}^L \frac{1}{2} \right) 2^{-\ell} c \norm{f}_\ast = \frac{1}{2} \veps \norm{f}_\ast.
        \end{align*}
\end{itemize}

In total, the procedure is comprised of downsampling operations, multigrid V-cycles and intermediate upsampling operations, all of which can be approximated arbitrarily well by CNNs: The downsampling of the diffusivity field is realized following Lemma~\ref{lmm:kappa_conv}, the V-cycles are realized as shown in Theorem~\ref{thm:cnn_multigrid} and the prolongation operators for upsampling are realized as in Remark~\ref{remark:cnn_restriction_prolongation}.
Hence, the procedure poses a concatenation of finitely many functions each representable by CNNs up to any accuracy. Using Lemma~\ref{lmm:approximation_function_combination} (with $M=M$ and $\veps = \frac{1}{2} \veps \norm{f}_\ast$) implies the existence of a CNN $\Psi$ such that for all $\bfy \in \Gamma$ and all right hand sides $\bff$, we have
\begin{align*}
    \hnorm{\Psi(\bfkappa_\bfy, \bff) - \bfv_L(\bfy)} &\leq \hnorm{\Psi(\bfkappa_\bfy, \bff) - \tilde{\bfv}_L(\bfy)} + \hnorm{\tilde{\bfv}_L(\bfy) - \bfv_L(\bfy)}\\
    &\leq \veps \norm{f}_\ast.
\end{align*}

The structure of the presented multilevel algorithm together with the parameter bounds from Theorem~\ref{thm:cnn_multigrid}, Lemma~\ref{lmm:kappa_conv} and Remark~\ref{remark:cnn_restriction_prolongation} implies that $\Psi$ is an $\MLNet$ architecture and that there exists a constant $C >0$ with
\begin{align*}
    M(\Psi) &\leq C \sum_{\ell=1}^{L-1} \ell + C \left(\log\left(\frac{2L}{\veps} \right) - L \log(2) \right) L\\
    &\leq C L \left(L + \log\left(\frac{1}{\veps} \right) + \log(2L) - L \log(2)\right)\\
    &\leq C L \log\left(\frac{1}{\veps} \right) + C L^2.
\end{align*}
\end{proofof}

\end{document}